\documentclass[twoside]{article}

\usepackage[accepted]{aistats2020}
\usepackage[utf8]{inputenc} %
\usepackage[T1]{fontenc}    %
\usepackage{hyperref}       %
\usepackage{url}            %
\usepackage{booktabs}       %
\usepackage{amsfonts}       %
\usepackage{nicefrac}       %
\usepackage{microtype}      %

\usepackage{listings}
\usepackage{autonum}
\usepackage{bm}
\usepackage{amsmath}
\usepackage{amssymb}
\usepackage{amsthm}
\usepackage{xcolor}
\usepackage{enumitem}
\usepackage{graphicx}
\usepackage{algpseudocode}
\usepackage{scalerel}
\usepackage{mathtools}
\usepackage{tabularx}
\usepackage{tablefootnote}
\usepackage{arydshln,leftidx}
\usepackage{ragged2e}  %
\newcolumntype{Y}{>{\RaggedRight\arraybackslash}X}
\algrenewcommand\textproc{}

\usepackage{color}
\definecolor{myfavblue}{rgb}{0.05, 0.2, 0.8}
\definecolor{keywords}{RGB}{255,0,90}
\definecolor{comments}{RGB}{0,0,113}
\definecolor{red}{RGB}{160,0,0}
\definecolor{green}{RGB}{0,150,0}
\definecolor{C0}{rgb}{0.12156862745098039, 0.4666666666666667, 0.7058823529411765}  %

\definecolor{myblue}{HTML}{3182bd}
\definecolor{myred}{HTML}{de2d26}

\usepackage{hyperref}
\usepackage{url}
\usepackage{algorithm, algpseudocode}%
\usepackage{nicefrac}  
\usepackage{relsize}
\usepackage{tablefootnote}

\usepackage[frozencache]{minted}
\usepackage{tikz}
\usetikzlibrary{fit,positioning}
\usetikzlibrary{calc,arrows.meta,positioning, arrows}
\usetikzlibrary{shapes.geometric}

\tikzset{
    every node/.style={font=\sffamily\Large},
    main node/.style={thick,circle,draw,font=\sffamily\huge}
}

\definecolor{mydarkblue}{rgb}{0,0.08,0.45}
\hypersetup{
    colorlinks=true,
    linkcolor=mydarkblue,
    citecolor=mydarkblue,
    filecolor=mydarkblue,
    urlcolor=mydarkblue
}

\floatname{algorithm}{Algorithm}

\algtext*{EndFunction}

\let\oldReturn\Return
\renewcommand{\Return}{\State\oldReturn}

\algnewcommand{\IfThenElse}[3]{%
  \State \algorithmicif\ #1\ \algorithmicthen\ #2\ \algorithmicelse\ #3}

\usepackage{scalerel,stackengine}
\stackMath
\newcommand\reallywidecheck[1]{%
\savestack{\tmpbox}{\stretchto{%
  \scaleto{%
    \scalerel*[\widthof{\ensuremath{#1}}]{\kern-.6pt\bigwedge\kern-.6pt}%
    {\rule[-\textheight/2]{1ex}{\textheight}}%
  }{\textheight}%
}{0.5ex}}%
\stackon[1pt]{#1}{\scalebox{-1}{\tmpbox}}%
}
\usepackage[square,sort,comma,numbers]{natbib}

\DeclareFontFamily{U}{matha}{\hyphenchar\font45}
\DeclareFontShape{U}{matha}{m}{n}{
      <5> <6> <7> <8> <9> <10> gen * matha
      <10.95> matha10 <12> <14.4> <17.28> <20.74> <24.88> matha12
      }{}
\DeclareSymbolFont{matha}{U}{matha}{m}{n}
\DeclareFontSubstitution{U}{matha}{m}{n}
\DeclareMathSymbol{\widecheck}{\widecheck}{matha}{'131}

\begin{document}

\twocolumn[

\aistatstitle{
Scalable Gradients for Stochastic Differential Equations
}

\aistatsauthor{
Xuechen Li$^*$ \And
Ting-Kam Leonard Wong \And
Ricky T. Q. Chen \And
David Duvenaud
}

\aistatsaddress{ Google Research \And  University of Toronto
\And University of Toronto \\
Vector Institute
\And University of Toronto \\
Vector Institute
}]

\newcommand{\figleft}{{\em (Left)}}
\newcommand{\figcenter}{{\em (Center)}}
\newcommand{\figright}{{\em (Right)}}
\newcommand{\figtop}{{\em (Top)}}
\newcommand{\figbottom}{{\em (Bottom)}}
\newcommand{\captiona}{{\em (a)}}
\newcommand{\captionb}{{\em (b)}}
\newcommand{\captionc}{{\em (c)}}
\newcommand{\captiond}{{\em (d)}}

\newcommand{\newterm}[1]{{\bf #1}}

\def\figref#1{figure~\ref{#1}}
\def\Figref#1{Figure~\ref{#1}}
\def\twofigref#1#2{figures \ref{#1} and \ref{#2}}
\def\quadfigref#1#2#3#4{figures \ref{#1}, \ref{#2}, \ref{#3} and \ref{#4}}
\def\secref#1{section~\ref{#1}}
\def\Secref#1{Section~\ref{#1}}
\def\twosecrefs#1#2{sections \ref{#1} and \ref{#2}}
\def\secrefs#1#2#3{sections \ref{#1}, \ref{#2} and \ref{#3}}
\def\eqref#1{equation~\ref{#1}}
\def\Eqref#1{Equation~\ref{#1}}
\def\plaineqref#1{\ref{#1}}
\def\chapref#1{chapter~\ref{#1}}
\def\Chapref#1{Chapter~\ref{#1}}
\def\rangechapref#1#2{chapters\ref{#1}--\ref{#2}}
\def\algref#1{algorithm~\ref{#1}}
\def\Algref#1{Algorithm~\ref{#1}}
\def\twoalgref#1#2{algorithms \ref{#1} and \ref{#2}}
\def\Twoalgref#1#2{Algorithms \ref{#1} and \ref{#2}}
\def\partref#1{part~\ref{#1}}
\def\Partref#1{Part~\ref{#1}}
\def\twopartref#1#2{parts \ref{#1} and \ref{#2}}

\def\ceil#1{\lceil #1 \rceil}
\def\floor#1{\lfloor #1 \rfloor}
\def\1{\bm{1}}
\newcommand{\train}{\mathcal{D}}
\newcommand{\valid}{\mathcal{D_{\mathrm{valid}}}}
\newcommand{\test}{\mathcal{D_{\mathrm{test}}}}

\def\eps{{\epsilon}}

\def\reta{{\textnormal{$\eta$}}}
\def\ra{{\textnormal{a}}}
\def\rb{{\textnormal{b}}}
\def\rc{{\textnormal{c}}}
\def\rd{{\textnormal{d}}}
\def\re{{\textnormal{e}}}
\def\rf{{\textnormal{f}}}
\def\rg{{\textnormal{g}}}
\def\rh{{\textnormal{h}}}
\def\ri{{\textnormal{i}}}
\def\rj{{\textnormal{j}}}
\def\rk{{\textnormal{k}}}
\def\rl{{\textnormal{l}}}
\def\rn{{\textnormal{n}}}
\def\ro{{\textnormal{o}}}
\def\rp{{\textnormal{p}}}
\def\rq{{\textnormal{q}}}
\def\rr{{\textnormal{r}}}
\def\rs{{\textnormal{s}}}
\def\rt{{\textnormal{t}}}
\def\ru{{\textnormal{u}}}
\def\rv{{\textnormal{v}}}
\def\rw{{\textnormal{w}}}
\def\rx{{\textnormal{x}}}
\def\ry{{\textnormal{y}}}
\def\rz{{\textnormal{z}}}

\def\rvepsilon{{\mathbf{\epsilon}}}
\def\rvtheta{{\mathbf{\theta}}}
\def\rva{{\mathbf{a}}}
\def\rvb{{\mathbf{b}}}
\def\rvc{{\mathbf{c}}}
\def\rvd{{\mathbf{d}}}
\def\rve{{\mathbf{e}}}
\def\rvf{{\mathbf{f}}}
\def\rvg{{\mathbf{g}}}
\def\rvh{{\mathbf{h}}}
\def\rvu{{\mathbf{i}}}
\def\rvj{{\mathbf{j}}}
\def\rvk{{\mathbf{k}}}
\def\rvl{{\mathbf{l}}}
\def\rvm{{\mathbf{m}}}
\def\rvn{{\mathbf{n}}}
\def\rvo{{\mathbf{o}}}
\def\rvp{{\mathbf{p}}}
\def\rvq{{\mathbf{q}}}
\def\rvr{{\mathbf{r}}}
\def\rvs{{\mathbf{s}}}
\def\rvt{{\mathbf{t}}}
\def\rvu{{\mathbf{u}}}
\def\rvv{{\mathbf{v}}}
\def\rvw{{\mathbf{w}}}
\def\rvx{{\mathbf{x}}}
\def\rvy{{\mathbf{y}}}
\def\rvz{{\mathbf{z}}}

\def\erva{{\textnormal{a}}}
\def\ervb{{\textnormal{b}}}
\def\ervc{{\textnormal{c}}}
\def\ervd{{\textnormal{d}}}
\def\erve{{\textnormal{e}}}
\def\ervf{{\textnormal{f}}}
\def\ervg{{\textnormal{g}}}
\def\ervh{{\textnormal{h}}}
\def\ervi{{\textnormal{i}}}
\def\ervj{{\textnormal{j}}}
\def\ervk{{\textnormal{k}}}
\def\ervl{{\textnormal{l}}}
\def\ervm{{\textnormal{m}}}
\def\ervn{{\textnormal{n}}}
\def\ervo{{\textnormal{o}}}
\def\ervp{{\textnormal{p}}}
\def\ervq{{\textnormal{q}}}
\def\ervr{{\textnormal{r}}}
\def\ervs{{\textnormal{s}}}
\def\ervt{{\textnormal{t}}}
\def\ervu{{\textnormal{u}}}
\def\ervv{{\textnormal{v}}}
\def\ervw{{\textnormal{w}}}
\def\ervx{{\textnormal{x}}}
\def\ervy{{\textnormal{y}}}
\def\ervz{{\textnormal{z}}}

\def\rmA{{\mathbf{A}}}
\def\rmB{{\mathbf{B}}}
\def\rmC{{\mathbf{C}}}
\def\rmD{{\mathbf{D}}}
\def\rmE{{\mathbf{E}}}
\def\rmF{{\mathbf{F}}}
\def\rmG{{\mathbf{G}}}
\def\rmH{{\mathbf{H}}}
\def\rmI{{\mathbf{I}}}
\def\rmJ{{\mathbf{J}}}
\def\rmK{{\mathbf{K}}}
\def\rmL{{\mathbf{L}}}
\def\rmM{{\mathbf{M}}}
\def\rmN{{\mathbf{N}}}
\def\rmO{{\mathbf{O}}}
\def\rmP{{\mathbf{P}}}
\def\rmQ{{\mathbf{Q}}}
\def\rmR{{\mathbf{R}}}
\def\rmS{{\mathbf{S}}}
\def\rmT{{\mathbf{T}}}
\def\rmU{{\mathbf{U}}}
\def\rmV{{\mathbf{V}}}
\def\rmW{{\mathbf{W}}}
\def\rmX{{\mathbf{X}}}
\def\rmY{{\mathbf{Y}}}
\def\rmZ{{\mathbf{Z}}}

\def\ermA{{\textnormal{A}}}
\def\ermB{{\textnormal{B}}}
\def\ermC{{\textnormal{C}}}
\def\ermD{{\textnormal{D}}}
\def\ermE{{\textnormal{E}}}
\def\ermF{{\textnormal{F}}}
\def\ermG{{\textnormal{G}}}
\def\ermH{{\textnormal{H}}}
\def\ermI{{\textnormal{I}}}
\def\ermJ{{\textnormal{J}}}
\def\ermK{{\textnormal{K}}}
\def\ermL{{\textnormal{L}}}
\def\ermM{{\textnormal{M}}}
\def\ermN{{\textnormal{N}}}
\def\ermO{{\textnormal{O}}}
\def\ermP{{\textnormal{P}}}
\def\ermQ{{\textnormal{Q}}}
\def\ermR{{\textnormal{R}}}
\def\ermS{{\textnormal{S}}}
\def\ermT{{\textnormal{T}}}
\def\ermU{{\textnormal{U}}}
\def\ermV{{\textnormal{V}}}
\def\ermW{{\textnormal{W}}}
\def\ermX{{\textnormal{X}}}
\def\ermY{{\textnormal{Y}}}
\def\ermZ{{\textnormal{Z}}}

\def\vzero{{\bm{0}}}
\def\vone{{\bm{1}}}
\def\vmu{{\bm{\mu}}}
\def\vtheta{{\bm{\theta}}}
\def\va{{\bm{a}}}
\def\vb{{\bm{b}}}
\def\vc{{\bm{c}}}
\def\vd{{\bm{d}}}
\def\ve{{\bm{e}}}
\def\vf{{\bm{f}}}
\def\vg{{\bm{g}}}
\def\vh{{\bm{h}}}
\def\vi{{\bm{i}}}
\def\vj{{\bm{j}}}
\def\vk{{\bm{k}}}
\def\vl{{\bm{l}}}
\def\vm{{\bm{m}}}
\def\vn{{\bm{n}}}
\def\vo{{\bm{o}}}
\def\vp{{\bm{p}}}
\def\vq{{\bm{q}}}
\def\vr{{\bm{r}}}
\def\vs{{\bm{s}}}
\def\vt{{\bm{t}}}
\def\vu{{\bm{u}}}
\def\vv{{\bm{v}}}
\def\vw{{\bm{w}}}
\def\vx{{\bm{x}}}
\def\vy{{\bm{y}}}
\def\vz{{\bm{z}}}

\def\evalpha{{\alpha}}
\def\evbeta{{\beta}}
\def\evepsilon{{\epsilon}}
\def\evlambda{{\lambda}}
\def\evomega{{\omega}}
\def\evmu{{\mu}}
\def\evpsi{{\psi}}
\def\evsigma{{\sigma}}
\def\evtheta{{\theta}}
\def\eva{{a}}
\def\evb{{b}}
\def\evc{{c}}
\def\evd{{d}}
\def\eve{{e}}
\def\evf{{f}}
\def\evg{{g}}
\def\evh{{h}}
\def\evi{{i}}
\def\evj{{j}}
\def\evk{{k}}
\def\evl{{l}}
\def\evm{{m}}
\def\evn{{n}}
\def\evo{{o}}
\def\evp{{p}}
\def\evq{{q}}
\def\evr{{r}}
\def\evs{{s}}
\def\evt{{t}}
\def\evu{{u}}
\def\evv{{v}}
\def\evw{{w}}
\def\evx{{x}}
\def\evy{{y}}
\def\evz{{z}}

\def\mA{{\bm{A}}}
\def\mB{{\bm{B}}}
\def\mC{{\bm{C}}}
\def\mD{{\bm{D}}}
\def\mE{{\bm{E}}}
\def\mF{{\bm{F}}}
\def\mG{{\bm{G}}}
\def\mH{{\bm{H}}}
\def\mI{{\bm{I}}}
\def\mJ{{\bm{J}}}
\def\mK{{\bm{K}}}
\def\mL{{\bm{L}}}
\def\mM{{\bm{M}}}
\def\mN{{\bm{N}}}
\def\mO{{\bm{O}}}
\def\mP{{\bm{P}}}
\def\mQ{{\bm{Q}}}
\def\mR{{\bm{R}}}
\def\mS{{\bm{S}}}
\def\mT{{\bm{T}}}
\def\mU{{\bm{U}}}
\def\mV{{\bm{V}}}
\def\mW{{\bm{W}}}
\def\mX{{\bm{X}}}
\def\mY{{\bm{Y}}}
\def\mZ{{\bm{Z}}}
\def\mBeta{{\bm{\beta}}}
\def\mPhi{{\bm{\Phi}}}
\def\mLambda{{\bm{\Lambda}}}
\def\mSigma{{\bm{\Sigma}}}

\newcommand{\tens}[1]{\bm{\mathsfit{#1}}}
\def\tA{{\tens{A}}}
\def\tB{{\tens{B}}}
\def\tC{{\tens{C}}}
\def\tD{{\tens{D}}}
\def\tE{{\tens{E}}}
\def\tF{{\tens{F}}}
\def\tG{{\tens{G}}}
\def\tH{{\tens{H}}}
\def\tI{{\tens{I}}}
\def\tJ{{\tens{J}}}
\def\tK{{\tens{K}}}
\def\tL{{\tens{L}}}
\def\tM{{\tens{M}}}
\def\tN{{\tens{N}}}
\def\tO{{\tens{O}}}
\def\tP{{\tens{P}}}
\def\tQ{{\tens{Q}}}
\def\tR{{\tens{R}}}
\def\tS{{\tens{S}}}
\def\tT{{\tens{T}}}
\def\tU{{\tens{U}}}
\def\tV{{\tens{V}}}
\def\tW{{\tens{W}}}
\def\tX{{\tens{X}}}
\def\tY{{\tens{Y}}}
\def\tZ{{\tens{Z}}}

\def\gA{{\mathcal{A}}}
\def\gB{{\mathcal{B}}}
\def\gC{{\mathcal{C}}}
\def\gD{{\mathcal{D}}}
\def\gE{{\mathcal{E}}}
\def\gF{{\mathcal{F}}}
\def\gG{{\mathcal{G}}}
\def\gH{{\mathcal{H}}}
\def\gI{{\mathcal{I}}}
\def\gJ{{\mathcal{J}}}
\def\gK{{\mathcal{K}}}
\def\gL{{\mathcal{L}}}
\def\gM{{\mathcal{M}}}
\def\gN{{\mathcal{N}}}
\def\gO{{\mathcal{O}}}
\def\gP{{\mathcal{P}}}
\def\gQ{{\mathcal{Q}}}
\def\gR{{\mathcal{R}}}
\def\gS{{\mathcal{S}}}
\def\gT{{\mathcal{T}}}
\def\gU{{\mathcal{U}}}
\def\gV{{\mathcal{V}}}
\def\gW{{\mathcal{W}}}
\def\gX{{\mathcal{X}}}
\def\gY{{\mathcal{Y}}}
\def\gZ{{\mathcal{Z}}}

\def\sA{{\mathbb{A}}}
\def\sB{{\mathbb{B}}}
\def\sC{{\mathbb{C}}}
\def\sD{{\mathbb{D}}}
\def\sF{{\mathbb{F}}}
\def\sG{{\mathbb{G}}}
\def\sH{{\mathbb{H}}}
\def\sI{{\mathbb{I}}}
\def\sJ{{\mathbb{J}}}
\def\sK{{\mathbb{K}}}
\def\sL{{\mathbb{L}}}
\def\sM{{\mathbb{M}}}
\def\sN{{\mathbb{N}}}
\def\sO{{\mathbb{O}}}
\def\sP{{\mathbb{P}}}
\def\sQ{{\mathbb{Q}}}
\def\sR{{\mathbb{R}}}
\def\sS{{\mathbb{S}}}
\def\sT{{\mathbb{T}}}
\def\sU{{\mathbb{U}}}
\def\sV{{\mathbb{V}}}
\def\sW{{\mathbb{W}}}
\def\sX{{\mathbb{X}}}
\def\sY{{\mathbb{Y}}}
\def\sZ{{\mathbb{Z}}}

\def\emLambda{{\Lambda}}
\def\emA{{A}}
\def\emB{{B}}
\def\emC{{C}}
\def\emD{{D}}
\def\emE{{E}}
\def\emF{{F}}
\def\emG{{G}}
\def\emH{{H}}
\def\emI{{I}}
\def\emJ{{J}}
\def\emK{{K}}
\def\emL{{L}}
\def\emM{{M}}
\def\emN{{N}}
\def\emO{{O}}
\def\emP{{P}}
\def\emQ{{Q}}
\def\emR{{R}}
\def\emS{{S}}
\def\emT{{T}}
\def\emU{{U}}
\def\emV{{V}}
\def\emW{{W}}
\def\emX{{X}}
\def\emY{{Y}}
\def\emZ{{Z}}
\def\emSigma{{\Sigma}}

\newcommand{\etens}[1]{\mathsfit{#1}}
\def\etLambda{{\etens{\Lambda}}}
\def\etA{{\etens{A}}}
\def\etB{{\etens{B}}}
\def\etC{{\etens{C}}}
\def\etD{{\etens{D}}}
\def\etE{{\etens{E}}}
\def\etF{{\etens{F}}}
\def\etG{{\etens{G}}}
\def\etH{{\etens{H}}}
\def\etI{{\etens{I}}}
\def\etJ{{\etens{J}}}
\def\etK{{\etens{K}}}
\def\etL{{\etens{L}}}
\def\etM{{\etens{M}}}
\def\etN{{\etens{N}}}
\def\etO{{\etens{O}}}
\def\etP{{\etens{P}}}
\def\etQ{{\etens{Q}}}
\def\etR{{\etens{R}}}
\def\etS{{\etens{S}}}
\def\etT{{\etens{T}}}
\def\etU{{\etens{U}}}
\def\etV{{\etens{V}}}
\def\etW{{\etens{W}}}
\def\etX{{\etens{X}}}
\def\etY{{\etens{Y}}}
\def\etZ{{\etens{Z}}}

\newcommand{\pdata}{p_{\rm{data}}}
\newcommand{\ptrain}{\hat{p}_{\rm{data}}}
\newcommand{\Ptrain}{\hat{P}_{\rm{data}}}
\newcommand{\pmodel}{p_{\rm{model}}}
\newcommand{\Pmodel}{P_{\rm{model}}}
\newcommand{\ptildemodel}{\tilde{p}_{\rm{model}}}
\newcommand{\pencode}{p_{\rm{encoder}}}
\newcommand{\pdecode}{p_{\rm{decoder}}}
\newcommand{\precons}{p_{\rm{reconstruct}}}

\newcommand{\laplace}{\mathrm{Laplace}} %

\newcommand{\E}{\mathbb{E}}
\newcommand{\Ls}{\mathcal{L}}
\newcommand{\R}{\mathbb{R}}
\newcommand{\emp}{\tilde{p}}
\newcommand{\lr}{\alpha}
\newcommand{\reg}{\lambda}
\newcommand{\rect}{\mathrm{rectifier}}
\newcommand{\softmax}{\mathrm{softmax}}
\newcommand{\sigmoid}{\sigma}
\newcommand{\softplus}{\zeta}
\newcommand{\KL}{D_{\mathrm{KL}}}
\newcommand{\Var}{\mathrm{Var}}
\newcommand{\standarderror}{\mathrm{SE}}
\newcommand{\Cov}{\mathrm{Cov}}
\newcommand{\normlzero}{L^0}
\newcommand{\normlone}{L^1}
\newcommand{\normltwo}{L^2}
\newcommand{\normlp}{L^p}
\newcommand{\normmax}{L^\infty}

\newcommand{\parents}{Pa} %

\let\ab\allowbreak

\newcommand{\eq}[1]{\begin{align}#1\end{align}}
\newcommand{\eqn}[1]{\begin{align*}#1\end{align*}}
\newcommand{\mypmatrix}[1]{\begin{pmatrix}#1\end{pmatrix}}

\newcommand{\dee}{\mathop{\mathrm{d}\!}}
\newcommand{\dt}{\,\dee t}
\newcommand{\de}{\,\dee e}
\newcommand{\ds}{\,\dee s}
\newcommand{\dx}{\,\dee x}
\newcommand{\dX}{\,\dee X}
\newcommand{\dy}{\,\dee y}
\newcommand{\dY}{\,\dee Y}
\newcommand{\dz}{\,\dee z}
\newcommand{\dZ}{\,\dee Z}
\newcommand{\dv}{\,\dee v}
\newcommand{\du}{\,\dee u}
\newcommand{\dw}{\,\dee w}
\newcommand{\dr}{\,\dee r}
\newcommand{\dB}{\,\dee B} %
\newcommand{\db}{\,\dee b}
\newcommand{\dW}{\,\dee W} %
\newcommand{\dtau}{\,\dee \tau}
\newcommand{\dmu}{\,\dee \mu}
\newcommand{\dnu}{\,\dee \nu}
\newcommand{\dzeta}{\,\dee \zeta}
\newcommand{\domega}{\,\dee \omega}

\newcommand*{\matr}[1]{\mathbfit{#1}}
\newcommand*{\tran}{^\top}
\newcommand*{\conj}[1]{\overline{#1}}
\newcommand*{\hermconj}{^{\mathsf{H}}}

\newcommand{\latent}{\vz}
\newcommand{\hidden}{\vh}
\newcommand{\obs}{x}
\newcommand{\sol}{z}
\newcommand{\obsdim}{D_x}
\newcommand{\latentdim}{D}
\newcommand{\solvefunc}{\textnormal{ODESolve}}
\newcommand{\tstart}{{t_\textnormal{0}}}
\newcommand{\tend}{{t_\textnormal{1}}}
\newcommand{\lograte}{\lambda}%
\newcommand{\method}{Latent ODE}
\newcommand{\cnfx}{\sol}
\newcommand{\adj}{a}

\newcommand*{\A}{\mathcal{A}}
\newcommand*{\B}{\mathcal{B}}
\newcommand*{\F}{\mathcal{F}}
\newcommand*{\V}{\mathbb{V}}
\newcommand*{\N}{\mathcal{N}}
\newcommand*{\TV}{\text{TV}}
\newcommand*{\LL}{\left}
\newcommand*{\RR}{\right}
\newcommand*{\tPhi}{\tilde{\Phi}}

\newcommand*{\bracks}[1]{\left(#1\right)}  %
\newcommand*{\abracks}[1]{\left\langle#1\right\rangle}  %
\newcommand*{\sbracks}[1]{\left[#1\right]}  %
\newcommand*{\norm}[1]{\left\|#1 \right\|}  %

\newcommand{\plim}{\mathrm{plim}}

\newtheorem{defi}{Definition}
\numberwithin{defi}{section}
\newtheorem{mycond}{Condition}

\newtheorem{fact}[defi]{Fact}
\newtheorem{algo}[defi]{Algorithm}
\newtheorem{cond}[defi]{Condition}
\newtheorem{prop}[defi]{Proposition}

\newtheorem{lemm}[defi]{Lemma}
\newtheorem{theo}[defi]{Theorem}
\newtheorem{coro}[defi]{Corollary}

\newtheorem{remark}{\textbf{Remark}}

\newtheorem{exam}{\textbf{Example}}

\newtheorem{assu}[defi]{Assumption}

\makeatletter
\DeclareRobustCommand\widecheck[1]{{\mathpalette\@widecheck{#1}}}
\def\@widecheck#1#2{%
    \setbox\z@\hbox{\m@th$#1#2$}%
    \setbox\tw@\hbox{\m@th$#1%
       \widehat{%
          \vrule\@width\z@\@height\ht\z@
          \vrule\@height\z@\@width\wd\z@}$}%
    \dp\tw@-\ht\z@
    \@tempdima\ht\z@ \advance\@tempdima2\ht\tw@ \divide\@tempdima\thr@@
    \setbox\tw@\hbox{%
       \raise\@tempdima\hbox{\scalebox{1}[-1]{\lower\@tempdima\box
\tw@}}}%
    {\ooalign{\box\tw@ \cr \box\z@}}}
\makeatother

\def\ssum{\mathsmaller{\sum}}
\def\sint{\mathsmaller{\int}}

\newcommand{\Exp}[1]{ \mathbb{E} \left[ #1 \right] }
\newcommand{\loss}{ \mathcal{L} }
\newcommand{\param}{ {\bm{\theta}} }

\newcommand{\mymid}{ {\text{mid}} }
\DeclarePairedDelimiter\abs{\lvert}{\rvert}

\begin{abstract}
The adjoint sensitivity method scalably computes gradients of solutions to ordinary differential equations.
We generalize this method to stochastic differential equations, allowing time-efficient and constant-memory computation of gradients with high-order adaptive solvers.
Specifically, we derive a stochastic differential equation whose solution is the gradient, a memory-efficient algorithm for caching noise, and conditions under which numerical solutions converge.
In addition, we combine our method with gradient-based stochastic variational inference for latent stochastic differential equations.
We use our method to fit stochastic dynamics defined by neural networks, achieving competitive performance on a 50-dimensional motion capture dataset.
\end{abstract}

\section{Introduction}
Deterministic dynamical systems can often be modeled by ordinary differential equations (ODEs).
The adjoint sensitivity method
can efficiently compute gradients of ODE solutions with constant memory cost.
This method was well-known in the physics, numerical analysis, and control communities for decades~\cite{pearlmutter1995gradient,pontryagin2018mathematical,andersson2013general,andersson2019casadi}.
Recently, it was combined with modern reverse-mode automatic differentiation packages, enabling ODEs with millions of parameters to be fit to data~\cite{chen2018neural} and allowing more flexible density estimation and time series models~\cite{grathwohl2019ffjord,rubanova2019latent,jia2019}.

Stochastic differential equations (SDEs) generalize ODEs, adding instantaneous noise to their dynamics~\cite{oksendal2003stochastic,sarkka2019applied,sarkka2013bayesian}.
They are a natural model for phenomena governed by many small and unobserved interactions, such as motion of molecules in a liquid~\cite{brown1828xxvii}, allele frequencies in a gene pool~\cite{ewens2012mathematical}, or prices in a market~\cite{shreve2004stochastic}.
Previous attempts on fitting SDEs mostly relied on methods with poor scaling properties.
The \textit{pathwise approach} \cite{gobet2005sensitivity,yang1991monte}, 
a form of forward-mode automatic differentiation, scales poorly in time with the number of parameters and states in the model.
On the other hand, simply differentiating through the operations of an SDE solver~\cite{giles2006smoking} scales poorly in memory.

In this work, we generalize the adjoint method to stochastic dynamics defined by SDEs.
We give a simple and practical algorithm for fitting SDEs with tens of thousands of parameters, while allowing the use of high-order adaptive time-stepping SDE solvers. We call this approach the \textit{stochastic adjoint sensitivity method}.

\begin{table}[h!]
	\setlength{\tabcolsep}{4pt}
	\label{tab:complexity}
	\centering
	\begin{tabular}{lll}
		\toprule
		\multicolumn{1}{c}{Method} & \multicolumn{1}{c}{Memory} & \multicolumn{1}{c}{Time} \\
		\midrule
		Forward pathwise \cite{yang1991monte,gobet2005sensitivity} & $\mathcal{O}(1)$ & $\mathcal{O}(L D)$ \\
		Backprop through solver \cite{giles2006smoking} & $\mathcal{O}(L)$ & $\mathcal{O}(L)$ \\
		{Stochastic adjoint (ours)} & $\mathcal{O}(1)$ & $\mathcal{O}(L \log L)$ \\
		\bottomrule
	\end{tabular}
	\vspace{-1mm}
	\caption{Asymptotic complexity comparison.
		$L$ is the number of steps used in a fixed-step solve, and $D$ is the number of state and parameters.
		Both memory and time are expressed in units of the cost of evaluating the drift and diffusion functions once each.
	}
\end{table}

There are two main difficulties in generalizing the adjoint formulation for ODEs to SDEs.
The first is mathematical: 
SDEs are defined using nonstandard integrals that usually rely on It\^o calculus.
The adjoint method requires
solving the dynamics backwards in time from the end state. %
However, it is not clear exactly what ``running the SDE backwards'' means in the context of stochastic calculus, and when it correctly reconstructs the forward trajectory.
We address this problem in Section~\ref{sec:adjoint}, deriving a backward Stratonovich SDE whose dynamics compute the necessary gradient.

The second difficulty is computational:
To retrace the steps, one needs to reconstruct the noise sampled on the forward pass, ideally without storing it.
In Section~\ref{sec:tree}, we give an algorithm that allows querying a Brownian motion sample at any time point arbitrarily-precisely, while only storing a single random seed.

We combine our adjoint approach with a gradient-based stochastic variational inference scheme for efficiently marginalizing over latent SDE models with arbitrary differentiable likelihoods.
This model family generalizes several existing families such as latent ODEs~\cite{chen2018neural,rubanova2019latent},
Gaussian state-space models~\cite{kitagawa1996linear, turner2010state},
and deep Kalman filters~\cite{krishnan2017structured}, and can naturally handle irregularly-sampled times series and missing observations.
We train latent SDEs on toy and real datasets, demonstrating competitive performance compared to existing approaches for dynamics modeling. 

\section{Background: Stochastic Flows}

\subsection{Adjoint Sensitivity Method}
\label{subsec:bg_adjoint_for_odes}
The adjoint sensitivity method is an efficient approach to solve control problems relying on the adjoint (co-state) system~\cite{pontryagin2018mathematical}. 
\citet{chen2018neural} used this method to compute the gradient with respect to parameters of a \emph{neural ODE}, which is a particular model among many others inspired by the theory of dynamical systems~\cite{lu2017beyond,chang2018reversible,ruthotto2018deep,haber2017stable,chang2017multi,li2017maximum,weinan2017proposal}.
The method, shown in Algorithm~\ref{algo:ode_adjoint}, is scalable, since the most costly computation is a vector-Jacobian product defining its backwards dynamics.
In addition, since the gradient is obtained by solving another ODE, no intermediate computation is stored as in the case of regular backpropagation~\cite{rumelhart1988learning}. 
\begin{figure*}[h!]
	\begin{minipage}[t]{0.49\linewidth}
		\begin{algorithm}[H]
			\centering
			\caption{ \footnotesize{ODE Adjoint Sensitivity} } \label{algo:ode_adjoint}
			\begin{algorithmic}
				\Require {
					Parameters $\theta$, start time $\tstart$, stop time $\tend$, 
					final state $\sol_\tend$, loss gradient ${\partial \loss}/{\sol_\tend}$,
					dynamics $f(\sol, t, \theta)$.
					\newline
				}
				\vspace{.75mm}
				\vspace{.75mm}
				\Function{\textnormal{$\overline f$}}{$[\sol_t, \adj_t, \cdot], \; t, \; \theta$}:
				\Comment{Augmented dynamics}
				\State $v = f(\sol_t, -t, \theta)$
				\State \textbf{return} $\![
				-v, \;
				\adj_t {\partial v}/{\partial \sol}, \;
				\adj_t {\partial v}/{\partial \theta} \,]$ 
				\EndFunction
				
				\vspace{24mm}
				\vspace{.75mm}
				$\!\!\!\!\!\!\!\!\!\!\!\!\left[
				\begin{array}{c}
				\!\!\!\! \sol_\tstart \!\!\!\! \\
				\!\!\!\! {\partial \loss}/{\partial \sol_\tstart} \!\!\!\!\\
				\!\!\!\!  {\partial \loss}/{\partial \theta} \!\!\!\!
				\end{array} \right] = \textnormal{\texttt{odeint}}
				\!\!\left( \left[
				\begin{array}{c}
				\!\!\!\! \sol_\tend \!\!\!\! \\
				\!\!\!\! {\partial \loss}/{\partial \sol_\tend} \!\!\!\!\\
				\!\!\!\!  {\vzero}_p \!\!\!\!
				\end{array} \right]\!\!,
				\textnormal{$\overline f$}, -\tend, -\tstart \right)$
				\Ensure ${\partial \loss}/{\partial \sol_\tstart}, {\partial \loss}/{\partial \theta }$
			\end{algorithmic}
		\end{algorithm}%
	\end{minipage}
	\hfill
	\begin{minipage}[t]{0.49\linewidth}
		\begin{algorithm}[H]
			\centering
			\caption{ \footnotesize{SDE Adjoint Sensitivity (Ours)} } \label{algo:overall}
			\begin{algorithmic}
				\Require {
					Parameters $\theta$, start time $\tstart$, stop time $\tend$, 
					final state $\sol_\tend$, loss gradient ${\partial \loss}/{\sol_\tend}$,
					drift $f(\sol, t, \theta)$, 
					{\color{myfavblue} diffusion $\sigma(\sol, t, \theta)$,
						Wiener process sample $w(t)$}.
				}
				\vspace{1.25mm}
				\Function{\textnormal{$\overline f$}}{$[\sol_t, \adj_t, \cdot], \; t, \; \theta$}:  \Comment{Augmented drift}
				\State $v = f(\sol_t, -t, \theta)$
				\State \textbf{return} $\![
				-v, \;
				\adj_t {\partial v}/{\partial \sol}, \;
				\adj_t {\partial v}/{\partial \theta}]$ 
				\EndFunction
				\vspace{.75mm}
				{\color{myfavblue}
					\Function{\textnormal{$\overline \sigma$}}{$[\sol_t, \adj_t, \cdot], \; t, \; \theta$}:
					\Comment{Augmented diffusion}
					\State $v = \sigma(\sol_t, -t, \theta)$
					\State \textbf{return} $\![
					-v,
					\adj_t {\partial v}/{\partial \sol},
					\adj_t {\partial v}/{\partial \theta}]$ 
					\EndFunction}
				\vspace{.75mm}
				{\color{myfavblue}
					\Function{\textnormal{$\overline w$}}{$t$}:
					\Comment{Replicated noise}
					\State \textbf{return} $\![{-w(-t)}, {-w(-t)}, {-w(-t)}]$ 
					\EndFunction}

				\vspace{.75mm}
				$\!\!\!\!\!\!\!\!\!\!\!\!\left[
				\begin{array}{c}
				\!\!\!\! \sol_\tstart \!\!\!\! \\
				\!\!\!\! {\partial \loss}/{\partial \sol_\tstart} \!\!\!\!\\
				\!\!\!\!  {\partial \loss}/{\partial \theta} \!\!\!\!
				\end{array} \right] = {\color{myfavblue}\textnormal{\texttt{sdeint}}}
				\!\!\left( \left[
				\begin{array}{c}
				\!\!\!\! \sol_\tend \!\!\!\! \\
				\!\!\!\! {\partial \loss}/{\partial \sol_\tend} \!\!\!\!\\
				\!\!\!\!  {\vzero}_p \!\!\!\!
				\end{array} \right] \!\!,
				\textnormal{$\overline f$}, {\color{myfavblue}\textnormal{$\overline \sigma$}, \textnormal{$\overline w$}}, 
				-\tend, -\tstart \right)$
				\Ensure ${\partial \loss}/{\partial \sol_\tstart}, {\partial \loss}/{\partial \theta }$
			\end{algorithmic}
		\end{algorithm}
	\end{minipage}
	\caption{
		Pseudocode of the (ODE) adjoint sensitivity method (\emph{left}), and our generalization to Stratonovich SDEs (\emph{right}).
		Differences are highlighted in blue. Square brackets denote vector concatenation. 
	}
\end{figure*}

\subsection{Stochastic Differential Equations}

Consider a filtered probability space $(\Omega, \F, \{ \F_t \}_{t \in \mathbb{T}}, P)$ on which an $m$-dimensional adapted Wiener process (or Brownian motion) $\{ W_t \}_{t \in\mathbb{T} }$ is defined. 
For a fixed terminal time $T > 0$, we denote by $\mathbb{T} = [0, T] $ the time horizon.
We denote the $i$th component of $W_t$ by $W_t^{(i)}$. 
Due to space constraint, we refer the read to Appendix~\ref{app:notation} for more on notation. 

A stochastic process $\{Z_t\}_{t \in\mathbb{T}}$ can be defined by an It\^o SDE
\eq{\label{eq:sde}
    Z_T \!=\!
        z_0 \!+\!  
        \int_0^T \! b ( Z_t, t ) \dt  \!+\!  
        \sum_{i=1}^m  \int_0^T \! \sigma_i(Z_t, t) \dW_t^{(i)},
}
where $z_0 \in \R^d$ is the starting state,
and $b: \R^d \times \R \to \R^d$ and $\sigma_i:\R^d \times \R \to \R^d$ are the drift and diffusion functions, respectively. 
For ease of presentation, we let $m=1$ in the following unless otherwise stated. Our contributions can be easily generalized to cases where $m > 1$.
Here, the second integral on the right hand side of (\ref{eq:sde}) is the It\^{o} stochastic integral~\cite{oksendal2003stochastic}. When the coefficients are globally Lipschitz in both the state and time, there exists a unique strong solution to the SDE~\cite{oksendal2003stochastic}. 

\subsection{Neural Stochastic Differential Equations}
\label{subsec:bg_neural_sde}
Similar to neural ODEs, one can consider drift and diffusion functions defined by neural networks, %
 a model known as the \textit{neural SDE}~\citep{tzen2019neural,tzen2019theoretical,liu2019neural,jia2019}.

Among works on neural SDEs, none has enabled an efficient training framework.
In particular, \citet{tzen2019neural} and \citet{liu2019neural} considered computing the gradient by simulating the forward dynamics of an explicit Jacobian matrix.
This Jacobian has size of either the square of the number of parameters, or the number of parameters times the number of states, building on the pathwise approach~\cite{gobet2005sensitivity,yang1991monte}.
In contrast, our approach only requires a small number of cheap vector-Jacobian products, independent of the dimension of the parameter and state vectors.
These vector-Jacobian products have the same asymptotic time cost as evaluating the drift and diffusion functions, and can be easily computed by modern automatic differentiation libraries~\citep{maclaurin2015autograd,paszke2017automatic,abadi2016tensorflow,frostig2018compiling}. 

\subsection{Backward Stratonovich Integral}
\label{subsec:bg_backward_stratonovich_integral}
Our stochastic adjoint sensitivity method involves stochastic processes running both forward and backward in time. The Stratonovich stochastic integral, due to its symmetry, gives nice expressions for the backward dynamics and is more convenient for our purpose.
Our results can also be straightforwardly applied to It\^o SDEs, relying on a simple conversion rule~(see e.g.~\cite[Sec. 2]{platen1999introduction}).

Following the treatment of Kunita \cite{kunita2019stochastic}, we introduce the forward and backward Stratonovich integrals. Let $\{ \F_{s, t}\}_{s \le t; s,t \in \mathbb{T}}$ be a \emph{two-sided filtration}, where $\F_{s, t}$ is the $\sigma$-algebra generated by $\{W_v - W_u: s \le u \le v \le t\}$ for $s, t \in \mathbb{T}$ such that $s \le t$. For a continuous semimartingale $\{ Y_t \}_{t\in \mathbb{T}}$ adapted to the forward filtration $\{\F_{0, t}\}_{t \in \mathbb{T}}$, the \emph{Stratonovich stochastic integral} is
\eq{
    \sint_0^T Y_t \circ \dW_t \! = \!\!
        \lim_{ |\Pi| \to 0 } \sum_{k=1}^N
        \tfrac{
            \bracks{
                Y_{t_k} + Y_{t_{k-1}}
            }
        }{2}
        \bracks{
            W_{t_k} - W_{t_{k-1}}
        },
}
where $\Pi = \{0 = t_0 < \cdots < t_N = T \}$ is a partition of the interval $\mathbb{T} = [0, T]$, $|\Pi| = \max_k t_{k} - t_{k-1}$ denotes the size of largest segment of the partition, and the limit is to be interpreted in the $L^2$ sense. The It\^{o} integral uses instead the left endpoint $Y_{t_k}$ rather than the average. In general, the It\^{o} and Stratonovich integrals differ by a term of finite variation.

To define the backward Stratonovich integral, we consider the \emph{backward Wiener process} $\{\widecheck{W}_t\}_{t \in \mathbb{T}}$ defined as $\widecheck{W}_t = W_t - W_T\; \text{for all} \; t \in \mathbb{T}$ that is adapted to the backward filtration $\{\F_{t, T}\}_{t \in \mathbb{T}}$. For a continuous semimartingale $\widecheck{Y}_t$ adapted to the backward filtration, the \emph{backward Stratonovich integral} is
\eq{
    \sint_s^T \widecheck{Y}_t \circ \dee \widecheck{W}_t \! = \!\!
        \lim_{|\Pi| \to 0} \sum_{k=1}^N
        \tfrac{
            \bracks{
                \widecheck{Y}_{t_k} + 
                \widecheck{Y}_{t_{k-1}}
            }
        }{2}
        \bracks{\widecheck{W}_{t_{k-1}} - \widecheck{W}_{t_k} },
}
where $\Pi = \{0 = t_N < \cdots < t_0 = T \}$ is the partition. 

\subsection{Stochastic Flow of Diffeomorphisms}
\label{subsec:bg_stochastic_flow}
It is well known that an ODE defines a flow of diffeomorphisms \cite{arnold1978ordinary}. Here we consider the stochastic analog for the Stratonovich SDE
\eq{
 Z_T =
    z_0 +
    \int_0^T b ( Z_t, t ) \dt +
    \int_0^T \sigma( Z_t, t ) \circ \dW_t. \label{eq:stratonovich_sde}
}
Throughout the paper, we assume that both $b$ and $\sigma$ have infinitely many bounded derivatives w.r.t.\ the state, and bounded first derivatives w.r.t.\ time, i.e.\ $b, \sigma \in C^{\infty, 1}_b$, and thus the SDE has a unique strong solution.
 Let $\Phi_{s, t}(z) := Z^{s, z}_t$ be the solution at time $t$ when the process is started at $z$ at time $s$. 
Given a realization of the Wiener process, this defines a collection of continuous maps $\mathcal{S} = \{ \Phi_{s, t}\}_{s \le t; s, t \in \mathbb{T}}$ from $\mathbb{R}^d$ to itself. 

The following theorem shows that these maps are diffeomorphisms (after choosing a suitable modification) and that they satisfy backward SDEs.
\begin{theo}[{\cite[Theorem 3.7.1]{kunita2019stochastic}}] \label{thm:stochastic_flow}
\begin{enumerate}
\item[(a)] With probability $1$, the collection $\mathcal{S} = \{ \Phi_{s, t}\}_{s \le t; s, t \in \mathbb{T}}$ satisfies the flow property
\eq{
\Phi_{s, t} (z) = \Phi_{u, t} ( \Phi_{s, u} (z) ), \quad s \leq u \leq t, \; z \in \R^d. \label{eq:flow.property}
}
Moreover, each $\Phi_{s, t}$ is a smooth diffeomorphism from $\mathbb{R}^d$ to itself. 
We thus call $\mathcal{S}$ the stochastic flow of diffeomorphisms generated by the SDE (\ref{eq:stratonovich_sde}).
\item[(b)] The backward flow $\widecheck{\Psi}_{s, t} := \Phi_{s, t}^{-1}$ satisfies the backward SDE:
\eq{
    \widecheck{\Psi}_{s, t}(z) = 
        z - 
        &
        \int_s^t b(\widecheck{\Psi}_{u, t}(z), u) \du -
        \\&
        \int_s^t \sigma(\widecheck{\Psi}_{u, t}(z), u) \circ \dee \widecheck{W}_u,
        \label{eq:backward_stratonovich_sde}
}
for all $z \in \R^d$ and $s, t \in \mathbb{T}$ such that $s \le t$.
\end{enumerate}
\end{theo}

The coefficients in (\ref{eq:stratonovich_sde}) and (\ref{eq:backward_stratonovich_sde}) differ by only a negative sign. This symmetry is due to our use of the Stratonovich integral (see Figure \ref{fig:stochastic.flow}).

\begin{figure}[ht]
\centering
\includegraphics[width=\columnwidth]{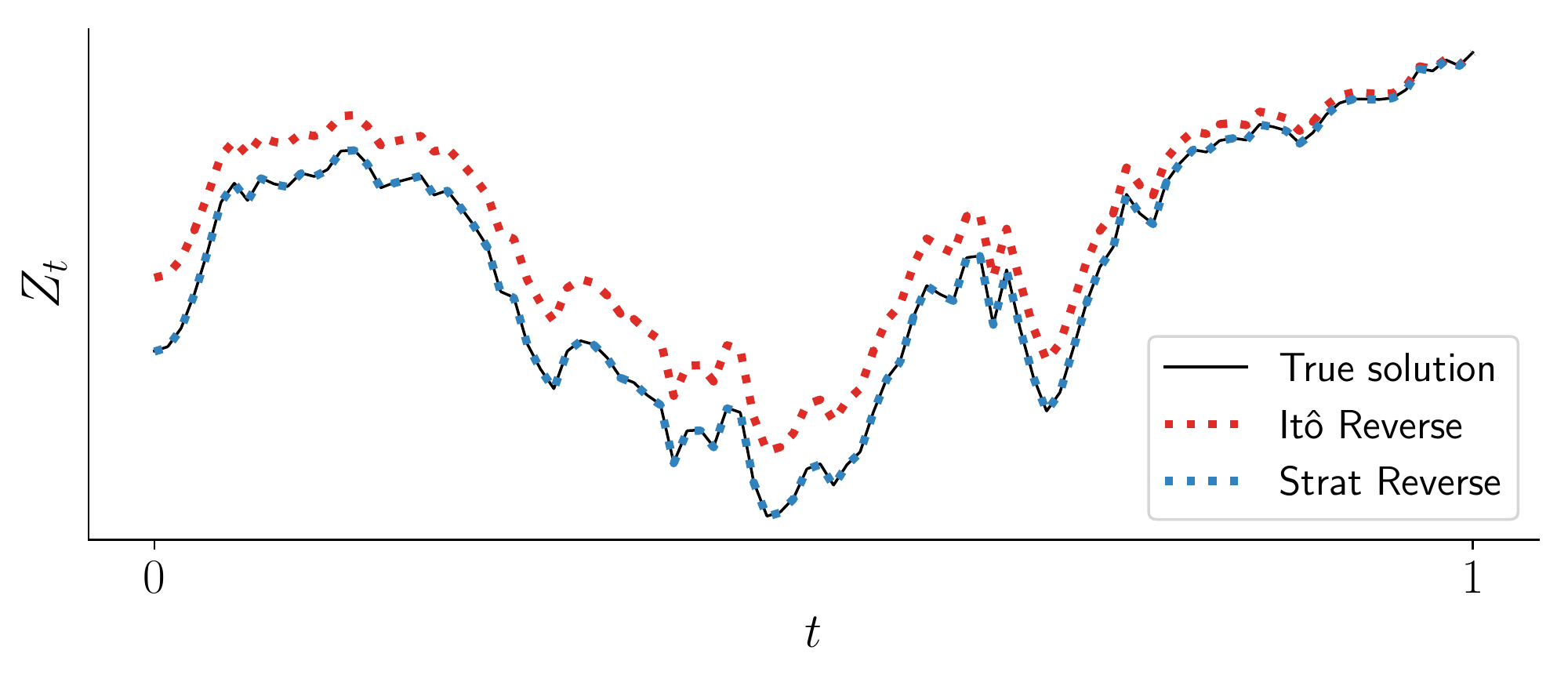}

\caption{
Negating the drift and diffusion functions for an It\^o SDE and simulating backwards from the end state gives the wrong reconstruction. 
Negating the drift and diffusion functions for the converted Stratonovich SDE gives the same path when simulated backwards. 
} \label{fig:stochastic.flow}
\vspace{-4mm}
\end{figure}

\section{Sensitivity via Stochastic Adjoint} \label{sec:adjoint}

We present our main contribution: a stochastic analog of the adjoint sensitivity method for SDEs. 
We use (\ref{eq:backward_stratonovich_sde}) to derive another backward Stratonovich SDE, which we call the \textit{stochastic adjoint process}.
The direct implication is a gradient computation algorithm that works by solving a set of dynamics in reverse time, and relies on cheap vector-Jacobian products without storing any intermediate quantities.

\subsection{Stochastic Adjoint Process}
The goal is to derive a stochastic adjoint process $\{\partial \loss / \partial Z_t\}_{t \in \mathbb{T}}$ that can be simulated by evaluating only vector-Jacobian products, where $\loss = \loss(Z_T)$ is a scalar loss of the terminal state from the forward flow $Z_T = \Phi_{0, T}(z_0)$.

We first derive a backward SDE for the process $\{\partial Z_T / \partial Z_t\}_{t \in \mathbb{T}}$, assuming that $Z_t = \widecheck{\Psi}_{t, T}(Z_T)$ follows the inverse flow from a deterministic end state $Z_T \in \R^d$ that does not depend on the realized Wiener process (Lemma~\ref{lemm:jacobian_dynamics}). 
We then extend to the case where $Z_T = \Phi_{0, T}(z_0)$ is obtained by the forward flow starting from a deterministic initial state $z_0 \in \R^d$ (Theorem \ref{thm:adjoint}).
This latter part is unconventional, and the resulting value cannot be interpreted as the solution to a backward SDE anymore due to loss of adaptedness. 
Instead, we will formulate the result with the \textit{It\^o map}~\cite{rogers2000diffusions2}. 
Finally, it is straightforward to extend the state $Z_t$ to include parameters of the drift and diffusion functions such that the desired gradient can be obtained for stochastic optimization; we comment on this step in Section~\ref{subsec:the_algorithm}.

We first present the SDE for the Jacobian matrix of the backward flow.

\begin{lemm}[Dynamics of ${\partial Z_T}/{\partial Z_t}$] \label{lemm:jacobian_dynamics}
Consider the stochastic flow generated by the backward SDE~(\ref{eq:backward_stratonovich_sde}) as in Theorem \ref{thm:stochastic_flow}(b). Letting $J_{s,t} (z) := \nabla \widecheck{\Psi}_{s, t}(z)$, we have
\eq{
    J_{s, t}(z) =
        I_d -
        &\int_s^t \nabla b(\widecheck{\Psi}_{r, t}(z), r) J_{r, t} (z) \dr - \\
        &\int_s^t \nabla \sigma( \widecheck{\Psi}_{r, t}(z), r ) J_{r, t} (z) \circ \dee \widecheck{W}_r,
    \label{eq:dynamics.J}
}
for all $s \le t$ and $z \in \R^d$.
Furthermore, letting $K_{s, t}(z) = [J_{s, t}(z)]^{-1}$, we have
\eq{
    K_{s, t}(z) =
        I_d + 
        &\int_s^t K_{r, t} (z) \nabla b(\widecheck{\Psi}_{r, t}(z), r) \dr + \\
        &\int_s^t K_{r, t} (z) \nabla \sigma( \widecheck{\Psi}_{r, t}(z), r ) \circ \dee \widecheck{W}_r,
    \label{eq:dynamics.K}
}
for all $s \le t$ and $z \in \R^d$.
\end{lemm}
The proof included in Appendix~\ref{app:jacobian_dynamics} relies on It\^o's lemma in the Stratonovich form~\cite[Theorem 2.4.1]{kunita2019stochastic}.
We stress that this lemma considers only the case where the endpoint $z$ is fixed and deterministic.

Now, we extend to the case where the endpoint is not deterministic, but rather computed from the forward flow.
To achieve this, we compose the state process and the loss function. 
Consider $A_{s, t}(z) = \partial \mathcal{L} (\Phi_{s, t}(z)) / \partial z$. 
The chain rule gives $A_{s, t}(z) = \nabla \mathcal{L}(\Phi_{s, t}(z)) \nabla \Phi_{s, t}(z)$. 
Let
\eq{
\label{eqn:A.Atilde}
\widecheck{A}_{s, t}(z) :=& A_{s, t}(\widecheck{\Psi}_{s, t}(z)) =\\& \nabla \mathcal{L}(z) \nabla \Phi_{s, t} (\widecheck{\Psi}_{s, t}(z)) = \nabla \mathcal{L}(z) K_{s, t}(z).
}
Note that $A_{s, t}(z) = \widecheck{A}_{s, t} (\Phi_{s, t}(z))$. 
Since $\nabla \mathcal{L}(z)$ is a constant, $(\widecheck{A}_{s, t}(z), \widecheck{\Psi}_{s, t}(z))$ satisfies the augmented backward SDE system
\eq{
\label{eqn:A.tilde.dynamics}
\begin{split}
\widecheck{A}_{s, t}(z) =&
    \nabla \mathcal{L}(z) + 
    \int_s^t \widecheck{A}_{r, t}(z) \nabla b(\widecheck{\Psi}_{r, t}(z), r) \dr + \\
    &\phantom{11111} \int_s^t \widecheck{A}_{r, t}(z) \nabla \sigma(\widecheck{\Psi}_{r, t}(z), r) \circ \dee \widecheck{W}_r, \\
\widecheck{\Psi}_{s, t}(z) =&
    z -
    \int_s^t b(\widecheck{\Psi}_{r, t}(z), r) \dr - \\
    &\phantom{11111} \int_s^t \sigma(\widecheck{\Psi}_{r, t}(z), r) \circ \dee \widecheck{W}_r.
\end{split}
}
Since the drift and diffusion functions of this augmented system are $C_b^{\infty, 1}$, the system has a unique strong solution.
Let $s=0$ and $t= T$.
Since (\ref{eqn:A.tilde.dynamics}) admits a strong solution, we may write
\eq{ \label{eqn:strong.solution}
\widecheck{A}_{0, T}(z) = \mathsf{F}(z, W_{\cdot}),
}
where $W_{\cdot} = \{W_t\}_{0 \leq t \leq T}$ denotes the path of the Wiener process and
\[
\mathsf{F} : \mathbb{R}^d \times C([0, 1], \mathbb{R}^m) \rightarrow \mathbb{R}^d
\]
is a deterministic measurable function (the {It\^{o} map})~\cite[Chapter V, Definition 10.9]{rogers2000diffusions2}. Intuitively, $\mathsf{F}$ can be thought as a black box that computes the solution to the backward SDE system~(\ref{eqn:A.tilde.dynamics}) given the position $z$ at time $T$ and the realized Wiener process sample. Similarly, we let $\mathsf{G}$ be the solution map for the forward flow~(\ref{eq:stratonovich_sde}). The next theorem follows immediately from (\ref{eqn:A.Atilde}) and the definition of $\mathsf{F}$.
\begin{theo} \label{thm:adjoint}
For $P$-almost all $\omega \in \Omega$, we have
\begin{equation} \label{eqn:adjoint.algorithm}
A_{0, T}(z) = \widecheck{A}_{0, T}( \mathsf{G}(z, W_{\cdot}) ) =
    \mathsf{F}(
        \mathsf{G}(z, W_{\cdot}) , W_{\cdot}
    ),
\end{equation}
where $\mathsf{G}(z, W_{\cdot}) = \Phi_{0, T}(z)$.
\end{theo}
\begin{proof}
This is a consequence of composing $A_{0, T}(z) = \widecheck{A}_{0, T} (\Phi_{0, T}(z))$ and (\ref{eqn:strong.solution}).
\end{proof}
This shows that one can obtain the gradient by ``composing'' the backward SDE system (\ref{eqn:A.tilde.dynamics}) with the original forward SDE (\ref{eq:stratonovich_sde}) and ends our continuous-time analysis.

\subsection{Numerical Approximation} \label{subsec:numerical_approximation}
In practice, we compute solutions to SDEs with numerical solvers $\mathsf{F}_h$ and $\mathsf{G}_h$, where $h = T / L$ denotes the mesh size of a fixed grid. The approximate algorithm thus outputs $\mathsf{F}_h(\mathsf{G}_h(z, W_{\cdot}), W_{\cdot})$. The following theorem provides sufficient conditions for convergence. 

\begin{theo} \label{thm:approx.scheme}
Suppose the schemes $\mathsf{F}_h$ and $\mathsf{G}_h$ satisfy the following conditions: (i) $\mathsf{F}_h(z, W_{\cdot}) \rightarrow \mathsf{F}(z, W_{\cdot})$ and $\mathsf{G}_h(z, W_{\cdot}) \rightarrow \mathsf{G}(z, W_{\cdot})$ in probability as $h \rightarrow 0$, and (ii) for any $M > 0$, we have $\sup_{| z | \le M} | \mathsf{F}_h(z, W_{\cdot}) - \mathsf{F}(z, W_{\cdot}) | \rightarrow 0$ in probability as $h \rightarrow 0$. Then, for any starting point $z$ of the forward flow, we have
\[
\mathsf{F}_h(\mathsf{G}_h(z, W_{\cdot}), W_{\cdot}) \rightarrow \mathsf{F}(\mathsf{G}(z, W_{\cdot}), W_{\cdot}) = A_{0, T}(z)
\]
in probability as $h \rightarrow 0$.
\end{theo}

See Appendix~\ref{app:approx.scheme} for the proof. 
Usual schemes such as the Euler-Maruyama scheme (more generally It\^o-Taylor schemes) converge pathwise (i.e. almost surely) from any fixed starting point~\cite{kloeden2007pathwise} and satisfies $(i)$. 
While $(ii)$ is strong, we note that the SDEs considered here have smooth coefficients, and thus their solutions enjoy nice regularity properties in the starting position. 
Therefore, it is reasonable to expect that the corresponding numerical schemes to also behave nicely as a function of \textit{both} the mesh size and the starting position. 
To the best of our knowledge, this property is not considered at all in the literature on numerical methods for SDEs (where the initial position is fixed), but is crucial in the proof of Theorem~\ref{thm:approx.scheme}. 
In Appendix \ref{app:euler_cond2}, we prove that condition $(ii)$ holds for the Euler-Maruyama scheme.
Detailed analysis for other schemes is beyond the scope of this paper.

\subsection{The Algorithm}\label{subsec:the_algorithm}
So far we have derived the gradient of the loss with respect to the initial state. 
We can extend these results to give gradients with respect to parameters of the drift and diffusion functions by treating them as an additional part of the state whose dynamics has zero drift and diffusion.
We summarize this in Algorithm~\ref{algo:overall}, assuming access only to a black-box solver \texttt{sdeint}.
All terms in the augmented dynamics, such as $a_t \partial f / \partial \theta$ and $a_t \partial \sigma / \partial \theta$ can be cheaply evaluated by calling $\texttt{vjp}(a_t, f, \theta)$ and $\texttt{vjp}(a_t, \sigma, \theta)$, respectively.	

\paragraph{Difficulties with non-diagonal diffusion.}
In principle, we can simulate the forward and backward adjoint dynamics with any high-order solver of choice. However, for general matrix-valued diffusion functions $\sigma$, to obtain a numerical solution with strong order\footnote{A numerical scheme is of strong order $p$ if $\Exp{ | X_T - X_{N\eta} | } \le C \eta^{p}$ for all $T > 0$, where $X_t$ and $X_{N\eta}$ are respectively the coupled true solution and numerical solution, $N$ and $\eta$ are respectively the iteration index and step size such that $N\eta = T$, and $C$ is independent of $\eta$.} beyond $1/2$, we need to simulate multiple integrals of the Wiener process such as $\int_0^t \int_0^s \dW_u^{(i)} \dW_s^{(j)}$, $i, j \in [m], i \ne j$.
These random variables are difficult to simulate and costly to approximate~\cite{wiktorsson2001joint}.

Fortunately, if we restrict our SDE to have diagonal noise, then even though the backward SDE for the stochastic adjoint will not in general have diagonal noise, it will satisfy a commutativity property~\cite{rossler2004runge}.
In that case, we can safely adopt certain numerical schemes of strong order 1.0 (e.g. Milstein~\cite{milstein1994numerical} and stochastic Runge-Kutta~\cite{rossler2010runge}) without approximating multiple integrals or the L\'evy area during simulation. 
We formally show this in Appendix~\ref{app:commutativity}.

One may also consider numerical schemes with high weak order~\cite{kloeden2013numerical}. 
However, analysis of this scenario is beyond the current scope.

\subsection{Software and Implementation}
We have implemented several common SDE solvers in PyTorch~\cite{paszke2017automatic} with adaptive time-stepping using a PI controller~\cite{burrage2004adaptive,ilie2015adaptive}. 
Following \texttt{torchdiffeq}~\cite{chen2018neural}, we have created a user-friendly subclass of \texttt{torch.autograd.Function} that facilitates gradient computation using our stochastic adjoint framework for SDEs that are subclasses of \texttt{torch.nn.Module}. 
We include a short code snippet covering the main idea of the stochastic adjoint in Appendix~\ref{app:adjoint_code}.
The complete codebase can be found at \url{https://github.com/google-research/torchsde}.

\section{Virtual Brownian Tree}\label{sec:tree}
Our formulation of the adjoint can be numerically integrated efficiently, since simulating its dynamics only requires evaluating cheap vector-Jacobian products, as opposed to whole Jacobians.
However, the backward-in-time nature introduces a new difficulty:
The same Wiener process sample path used in the forward pass must be queried again during the backward pass.
Na\"ively storing Brownian motion increments 
implies a large memory consumption and complicates the usage of adaptive time-stepping integrators, where the evaluation times in the backward pass may be different from those in the forward pass. 

To overcome this issue, we combine Brownian trees with splittable pseudorandom number generators (PRNGs) to give an algorithm that can query values of a Wiener process sample path at arbitrary times.
This algorithm, which we call the \emph{virtual Brownian tree}, has $\mathcal{O}(1)$ memory cost, and time cost logarithmic with respect to the inverse error tolerance.

\begin{figure}[ht]
{\includegraphics[width=\linewidth, clip, trim=0mm 0mm 0mm 0mm]{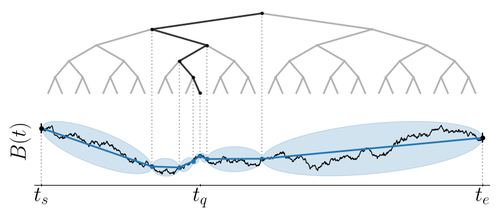}}
\caption{
Evaluating a Brownian motion sample at time $t_q$ using a virtual Brownian tree.
Our algorithm repeatedly bisects the interval, sampling from a Brownian bridge at each halving to determine intermediate values.
Each call to the random number generator uses a unique key whose value depends on the path taken to reach it.
}
\end{figure}

\subsection{Brownian Bridges and Brownian Trees}
L\'evy's \textit{Brownian bridge}~\citep{revuz2013continuous} states that given a start time $t_s$ and end time $t_e$ along with their respective Wiener process values $w_s$ and $w_e$, the marginal of the process at time $t \in (t_s, t_e)$ is a normal distribution:
\eq{
    \mathcal{N}
    \bracks{
        \frac{(t_e - t) w_s + (t - t_s) w_e}{t_e - t_s},
        \frac{(t_e - t) (t - t_s) }{t_e - t_s} I_d
    } \label{eq:brownian_bridge}
    .
}
We can recursively apply this formula to evaluate the process at the midpoint of any two distinct timestamps where the values are already known. Constructing the whole sample path of a Wiener process in this manner results in what is known as the \textit{Brownian tree}~\citep{gaines1997variable}.
Storing this tree would be memory-intensive, but we show how to reconstruct any node in this tree as desired.

\subsection{Brownian Trees using Splittable Seeds}
We assume access to a splittable PRNG~\citep{claessen2013splittable}, which has an operation \texttt{split} that deterministically generates two keys from an existing key.
Given a key, the function \texttt{BrownianBridge} samples deterministically from~(\ref{eq:brownian_bridge}). 
To obtain the Wiener process value at a specific time, we must first know or sample the values at the initial and terminal times.
Then, the virtual Brownian tree recursively samples from the midpoint of Brownian bridges, each sample using a key split from that of its parent node.
The algorithm terminates when the most recently sampled time is close enough to the desired time. 
We outline the full procedure in Algorithm~\ref{algo:brownian_tree}.
\begin{algorithm}[h]
\centering
\caption{Virtual Brownian Tree}
\label{algo:brownian_tree}
\begin{algorithmic}
    \Require {
        Seed $s$, query time $t$, error tolerance $\epsilon$,
        start time $t_s$, start state $w_s$, 
        end time $t_e$, end state $w_e$.
    }

    \vspace{.75mm}
    \State $t_m = (t_s + t_e) / 2$
    \State $s_m, s_l, s_r = \texttt{split}(s, \texttt{children=3})$
    \State $w_m = \texttt{BrownianBridge}(t_s, w_s, t_e, w_e, t_m, s_m)$ 
    \While{$\abs{t - t_m } > \epsilon$}
    \If{$t < t_m$} $t_e, x_e, s = t_m, w_m, s_l$
    \Else{} $t_s, x_s, s = t_m, w_m, s_r$
    \EndIf
    \State $t_m = (t_s + t_e) / 2$
    \State $s_m, s_l, s_r = \texttt{split}(s, \texttt{children=3}) $
    \State $w_m = \texttt{BrownianBridge}(t_s, w_s, t_e, w_e, t_m, s_m)$ 
    \EndWhile
    \Return{$w_m$}
\end{algorithmic}
\end{algorithm}
This algorithm has constant memory cost.
For a fixed-step-size solver taking $L$ steps, the tolerance that the tree will need to be queried at scales as $1 / L$.
Thus the per-step time complexity scales as $\log L$.
Our implementation uses an efficient \textit{count-based PRNG}~\cite{salmon2011parallel} which avoids passing large random states, and instead simply passes integers.
Table~\ref{tab:complexity} compares the asymptotic time complexity of this approach against existing alternatives.

\section{Latent Stochastic Differential Equations}
The algorithms presented in Sections~\ref{sec:adjoint} and \ref{sec:tree} allow us to efficiently compute gradients of scalar objectives with respect to SDE parameters, letting us fit SDEs to data.
This raises the question: Which loss to optimize?

Simply fitting SDE parameters to maximize likelihood will in general cause overfitting, and will result in the diffusion function going to zero.
In this section, we show how to do efficient variational inference in SDE models, and optimize the marginal log-likelihood to fit both prior (hyper-)parameters and the parameters of a tractable approximate posterior over functions.

\begin{figure*}[ht]
\begin{minipage}[t]{0.46\textwidth}
\centering
\scalebox{0.75}{
\begin{tikzpicture}[->,>=stealth',shorten >=1pt,auto,node distance=1.5cm,main node/.style={thick,circle,draw,font=\sffamily\Large}]
\tikzstyle{main}=[circle, minimum size = 10mm, thick, draw =black!80, node distance = 12mm]
\tikzstyle{latent}=[diamond, minimum size = 10mm, thick, draw =black!80, node distance = 12mm]
\tikzstyle{observation}=[circle, minimum size = 10mm, thick, draw =black!80, node distance = 12mm]
\tikzstyle{rnn}=[circle, minimum size = 3mm, thick, draw =black!80, node distance = 16mm]
\tikzstyle{connect}=[-latex]
\tikzstyle{box}=[rectangle, draw=black!100]
  \node[main, fill = white!100] (z_0) [] { $z_0$ };
  \node[latent] (z_{t_1}) [right=of z_0] { $z_{t_1}$ };
  \node[latent] (z_{t_2}) [right=of z_{t_1}] {$z_{t_2}$};
  \node[latent] (z_{t_n}) [right=of z_{t_2}] { $z_{t_n}$ };
  \node (zdots) at ($(z_{t_2})!.5!(z_{t_n})$) {\ldots};
  \node[circle, minimum size = 10mm, ultra thick, draw =black!80, node distance = 10mm, color=myred!90, fill=myred!10] (b) [above=of z_{t_2}] { $w(\cdot)$ };

  \node[rectangle, inner sep=4mm, fit= (z_{t_1}) (z_{t_2}) (z_{t_n}),label=above right:{\large \textbf{\texttt{\color{myblue} SDESolve}}}, xshift=6mm, yshift=-4mm] {};
  \node[rectangle, inner sep=6mm, draw=black!100, fit = (z_{t_1}) (z_{t_2}) (z_{t_n}), rounded corners=0.5cm, myblue!95, ultra thick] {};
  \node[diamond, minimum size = 10mm, thick, node distance = 8mm] (h_1) [below=of z_{t_1}] {};
  \node[diamond, minimum size = 10mm, thick, node distance = 8mm] (h_2) [below=of z_{t_2}] {};
  \node[diamond, minimum size = 10mm, thick, node distance = 8mm] (h_n) [below=of z_{t_n}] {};

  \node[circle, minimum size = 10mm, thick, node distance = 10mm, draw=black!80] (h_0) [below=of z_0] {$\theta$};

  \node[observation, node distance=0mm] (x_{t_1}) [below=of h_1] {$x_{t_1}$};
  \node[observation, node distance=0mm] (x_{t_2}) [below=of h_2] {$x_{t_2}$};
  \node[observation, node distance=0mm] (x_{t_n}) [below=of h_n] {$x_{t_n}$};
  \node (xdots) at ($(x_{t_2})!.5!(x_{t_n})$) {\ldots};

  \path (z_0) edge [connect] (z_{t_1})
        (z_{t_1}) edge [connect] (z_{t_2})
    (z_{t_2}) edge [connect] (zdots)
    (zdots) edge [connect] (z_{t_n})
    (b) edge [connect] (z_{t_n})
    (b) edge [connect] (z_{t_2})
    (b) edge [connect] (z_{t_1})

    (z_{t_n}) edge [connect] (x_{t_n})
    (z_{t_2}) edge [connect] (x_{t_2})
    (z_{t_1}) edge [connect] (x_{t_1})

    (h_0) edge [connect] (z_{t_1})
    (h_0) edge [connect] (z_{t_2})
    (h_0) edge [connect] (z_{t_n})
    ;
\end{tikzpicture}
}

(a) Generation
\end{minipage}
\hfill
\begin{minipage}[t]{0.46\textwidth}
	\centering
	\scalebox{0.75}{
		\begin{tikzpicture}[->,>=stealth',shorten >=1pt,auto,node distance=1.5cm,main node/.style={thick,circle,draw,font=\sffamily\Large}]
		\tikzstyle{main}=[circle, minimum size = 10mm, thick, draw =black!80, node distance = 12mm]
		\tikzstyle{latent}=[diamond, minimum size = 10mm, thick, draw =black!80, node distance = 12mm]
		\tikzstyle{observation}=[circle, minimum size = 10mm, thick, draw =black!80, node distance = 12mm]
		\tikzstyle{rnn}=[circle, minimum size = 3mm, thick, draw =black!80, node distance = 16mm]
		\tikzstyle{connect}=[-latex]
		\tikzstyle{box}=[rectangle, draw=black!100]
		\node[main, fill = white!100] (z_0) [] { $z_0$ };
		\node[latent] (z_{t_1}) [right=of z_0] { $z_{t_1}$ };
		\node[latent] (z_{t_2}) [right=of z_{t_1}] {$z_{t_2}$};
		\node[latent] (z_{t_n}) [right=of z_{t_2}] { $z_{t_n}$ };
		\node (zdots) at ($(z_{t_2})!.5!(z_{t_n})$) {\ldots};

		\node[circle, minimum size = 10mm, ultra thick, draw =black!80, node distance = 10mm, color=myred!90, fill=myred!10] (b) [above=of z_{t_2}] { $w(\cdot)$ };
		\node[rectangle, inner sep=4mm, fit= (z_{t_1}) (z_{t_2}) (z_{t_n}),label=above right:{\large \textbf{\texttt{\color{myblue} SDESolve}}}, xshift=6mm, yshift=-4mm] {};
		\node[rectangle, inner sep=6mm, draw=black!100, fit = (z_{t_1}) (z_{t_2}) (z_{t_n}), rounded corners=0.5cm, myblue!95, ultra thick] {};

		\node[diamond, minimum size = 10mm, thick, draw =black!80, node distance = 8mm] (mu_sigma) [below=of z_0] {$\phi$};
		\node[diamond, minimum size = 10mm, thick,  node distance = 8mm] (h_1) [below=of z_{t_1}] {};
		\node[diamond, minimum size = 10mm, thick,  node distance = 8mm] (h_2) [below=of z_{t_2}] {};
		\node[diamond, minimum size = 10mm, thick,  node distance = 8mm] (h_n) [below=of z_{t_n}] {};

		\node[observation, fill = black!10, node distance=0mm] (x_{t_1}) [below=of h_1] {$x_{t_1}$};
		\node[observation, fill = black!10, node distance=0mm] (x_{t_2}) [below=of h_2] {$x_{t_2}$};
		\node[observation, fill = black!10, node distance=0mm] (x_{t_n}) [below=of h_n] {$x_{t_n}$};
		\node (xdots) at ($(x_{t_2})!.5!(x_{t_n})$) {\ldots};

		\path (z_0) edge [connect] (z_{t_1})
		(z_{t_1}) edge [connect] (z_{t_2})
		(z_{t_2}) edge [connect] (zdots)
		(zdots) edge [connect] (z_{t_n})
		(b) edge [connect] (z_{t_n})
		(b) edge [connect] (z_{t_2})
		(b) edge [connect] (z_{t_1})
		
		(x_{t_1}) edge [connect, bend right = 10] (mu_sigma)
		(x_{t_2}) edge [connect, bend right = 10] (mu_sigma)
		(x_{t_n}) edge [connect, bend right = 10] (mu_sigma)
		
		(x_{t_1}) edge [connect] (z_0)
		(x_{t_2}) edge [connect] (z_0)
		(x_{t_n}) edge [connect] (z_0)
		
		(mu_sigma) edge [connect] (z_{t_1})
		(mu_sigma) edge [connect] (z_{t_2})
		(mu_sigma) edge [connect] (z_{t_n})
		;
		\end{tikzpicture}
	}
	
	(b) Recognition
\end{minipage}

\caption{
  Graphical models for the generative process (decoder) and recognition network (encoder) of the latent stochastic differential equation model.
  This model can be viewed as a variational autoencoder with infinite-dimensional noise.
  Red circles represent entire function draws from Brownian motion. 
  Given the initial state $z_0$ and a Brownian motion sample path $w(\cdot)$, the intermediate states $z_{t_1}, \dots, z_{t_n}$
  are deterministically approximated by a numerical SDE solver.
} \label{fig:latent_sde}
\end{figure*}

In particular, we can parameterize both a prior over functions and an approximate posterior using SDEs:
\begin{align}
\dee \tilde{Z_t} &= h_\theta(\tilde{Z_t}, t) \dt + \sigma(\tilde{Z_t}, t) \dee W_t, \label{eq:prior} \tag{prior}\\
\dZ_t &= h_\phi(Z_t, t) \dt + \sigma(Z_t, t) \dW_t, \label{eq:posterior} \tag{approx. post.}
\end{align}
where $h_\theta, h_\phi$, and $\sigma$
are Lipschitz in both arguments, and both processes have the same starting value: $ \tilde{Z_0} = Z_0 = z_0 \in \R^d$.

If both processes share the same diffusion function $\sigma$, then the KL divergence between them is finite (under additional mild regularity conditions; see Appendix~\ref{app:latent_sde_adjoint}), and can be estimated by sampling paths from the approximate posterior process.
Then, the evidence lower bound (ELBO) can be written as:
\begin{align}
\label{eq:variational_free_energy}
\log p(x_1, x_2, \dots, x_N | \theta) \geq \qquad\qquad\qquad\qquad\qquad\quad \\
\mathbb{E}_{Z_t} \left[
    \sum_{i=1}^N \log p( x_{t_i} | z_{t_{i}} )
    -\int_0^T \frac{1}{2} |u(z_t, t)|^2 \dt \right], \nonumber
\end{align}
where $u: \R^d \times [0, T] \to \R^m$ satisfies
\begin{align}
\sigma(z, t) u(z, t) = h_\phi(z, t) - h_\theta(z, t),
\end{align}
and the expectation is taken over the approximate posterior process defined by~(\ref{eq:posterior}).
The likelihoods of observations $x_1, \dots, x_N$ at times $t_1, \dots, t_N$ depend only on latent states $z_t$ at corresponding times.

To compute the gradient with respect to prior parameters $\theta$ and variational parameters $\phi$, we need only augment the forward SDE with an extra scalar variable whose drift is $\frac{1}{2} |u(Z_t, t)|^2$ and diffusion is zero.
The backward dynamics can be derived analogously using~(\ref{eqn:A.tilde.dynamics}). 
We include a detailed derivation in Appendix~\ref{app:latent_sde_adjoint}.
Thus, a stochastic estimate of the gradients of the loss w.r.t.\ all parameters can be computed in a single pair of forward and backward SDE solves.

The variational parameters $\phi$ can either be optimized individually for each sequence, or if multiple time series are sharing parameters, then an encoder network can be trained to input the observations and output $\phi$.
This architecture, shown in Figure~\ref{fig:latent_sde}, can be viewed as an infinite-dimensional Variational AutoEncoder (VAE)~\cite{kingma2013auto,rezende2014stochastic}, whose latent is an SDE-induced stochastic process.
We may generalize the above to cases where the diffusion is parameterized, which is then analogous to learning the prior of the latent code in VAEs.

\begin{figure*}[t]\label{fig:numerical_study}
\begin{minipage}[t]{0.33\linewidth}
\centering
{\includegraphics[width=0.98\textwidth]{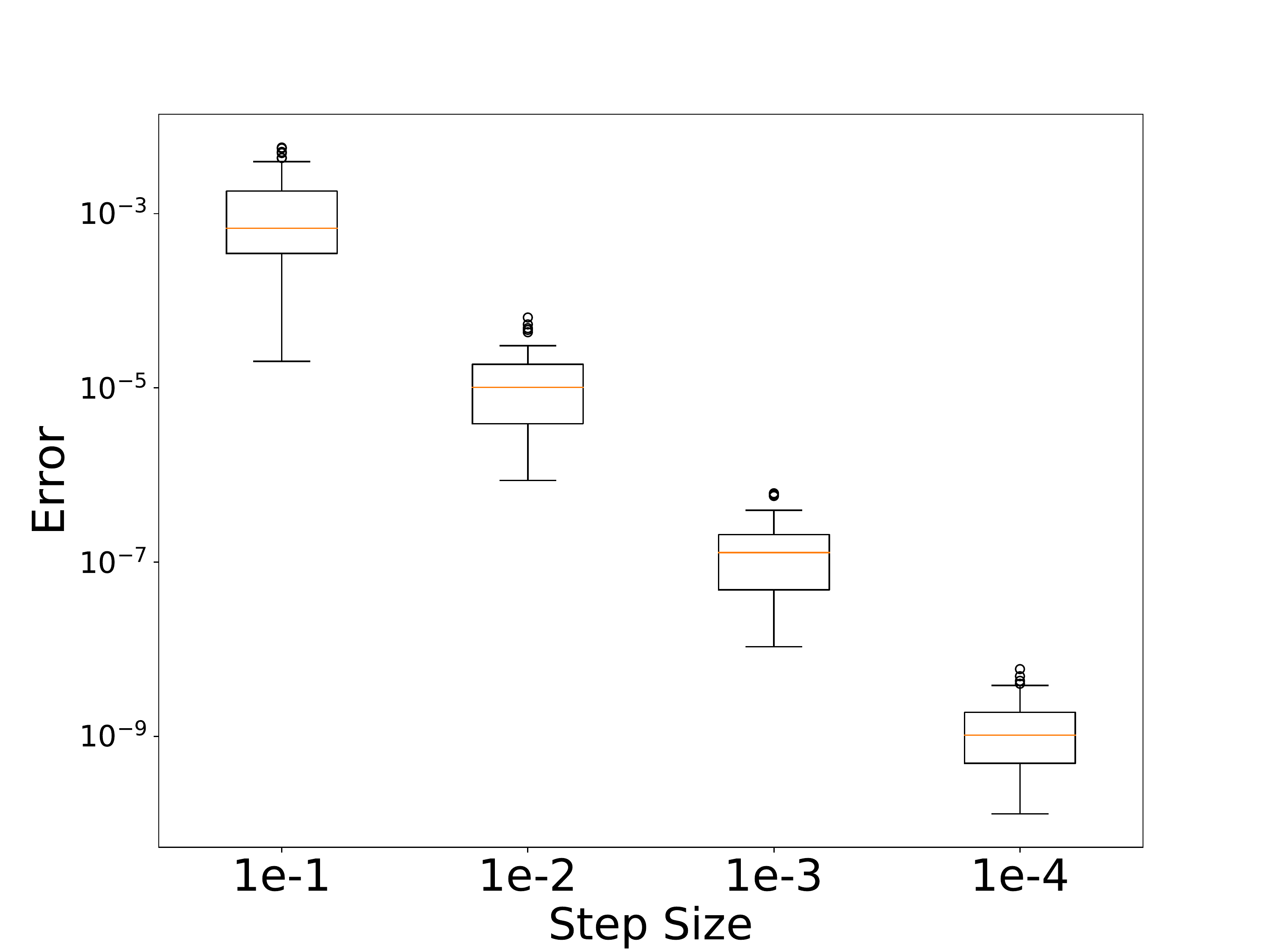}} \\ \vspace{-0.10cm}
(a) \footnotesize{Fixed Step Size vs Error}
\end{minipage}
\begin{minipage}[t]{0.33\linewidth}
\centering
\includegraphics[width=0.98\textwidth]{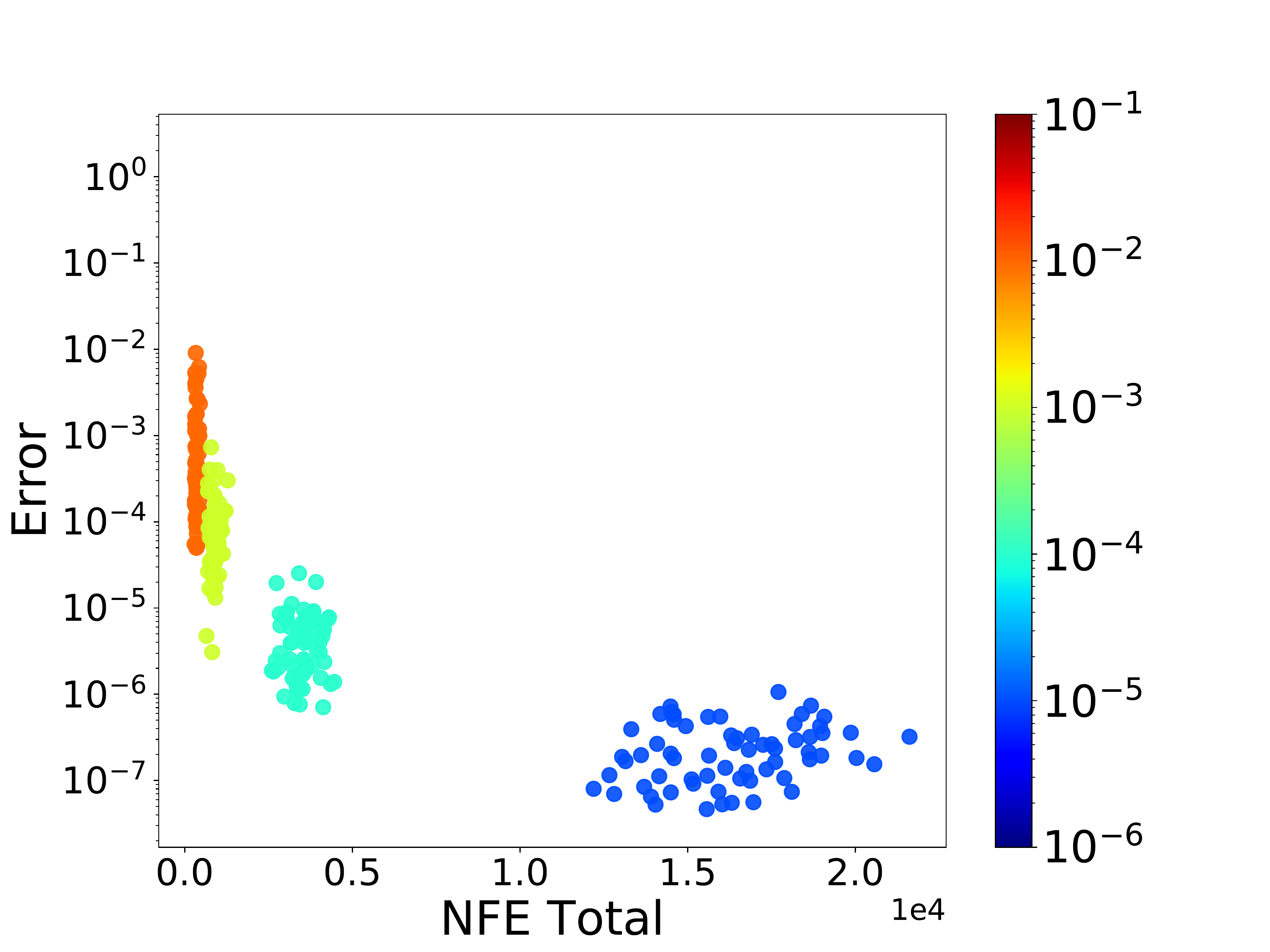} \\ \vspace{-0.10cm}
(b) \footnotesize{Forward NFE vs Error}
\end{minipage}
\begin{minipage}[t]{0.33\linewidth}
\centering
\includegraphics[width=0.98\textwidth]{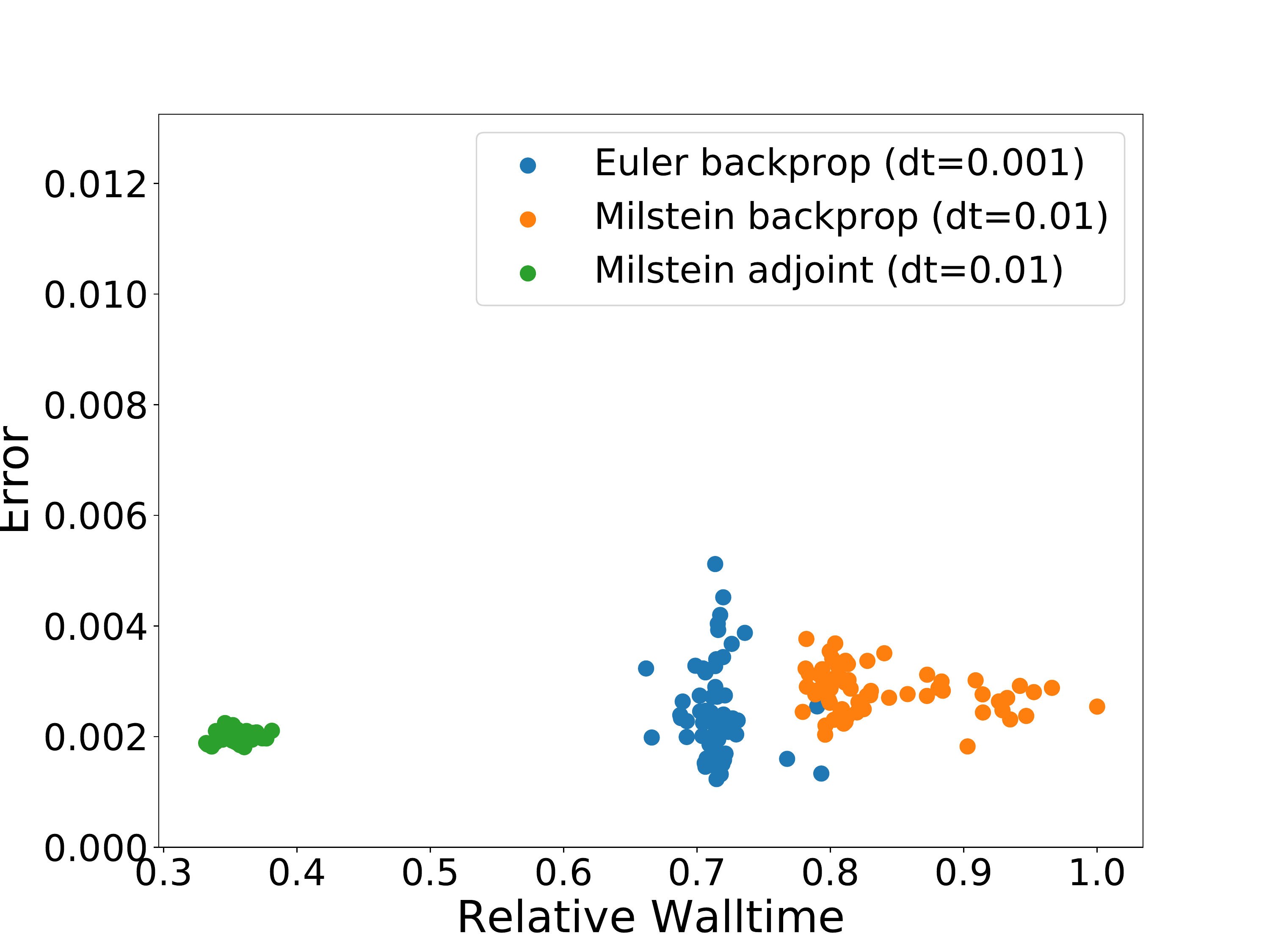} \\ \vspace{-0.10cm}
(c) \footnotesize{Efficiency Comparison}
\end{minipage}
\caption{
(a) Same fixed step size used in both forward and reverse simulation. Boxplot generated by repeating the experiment with different Brownian motion sample paths 64 times.
(b) Colors of dots represent tolerance levels and correspond to the colorbar on the right. 
Only \texttt{atol} was varied and \texttt{rtol} was set to $0$.
}
\end{figure*}

\section{Related Work}

\paragraph{Sensitivity Analysis for SDEs.}
Gradient computation is closely related to sensitivity analysis. 
Computing gradients with respect to parameters of vector fields of an SDE has been extensively studied in the stochastic control literature~\citep{kushner2013numerical}.
In particular, for low dimensional problems, this is done effectively using dynamic programming~\cite{baxterr2001infinite} and finite differences~\cite{glasserman1992some,l1994convergence}.
However, both approaches scale poorly with the dimensionality of the parameter vector. 

Analogous to REINFORCE (or the score-function estimator)~\citep{williams1992simple,kleijnen1996optimization,glynn1990likelihood}, \citet{yang1991monte} considered deriving the gradient as $\nabla \Exp{\mathcal{L} (Z_T) } = \Exp{ \mathcal{L}(Z_T) H }$ for some random variable $H$. However, $H$ usually depends on the density of $Z_T$ with respect to the Lebesgue measure which can be difficult to compute.
\citet{gobet2005sensitivity} extended this approach by weakening a non-degeneracy condition using Mallianvin calculus~\cite{nourdin2012normal}.

Closely related to the current approach is the pathwise method~\cite{yang1991monte}, which is also a continuous-time analog of the reparameterization trick~\citep{kingma2013auto,rezende2014stochastic}. Existing methods in this regime~\citep{tzen2019neural,gobet2005sensitivity,liu2019neural} all require simulating a (forward) SDE where each step requires computing entire Jacobian matrices. This computational cost is prohibitive for high-dimensional systems with a large number of parameters.

Based on the Euler discretization, \citet{giles2006smoking} considered %
simply performing reverse-mode automatic differentiation through all intermediate steps.
They named this method the \textit{adjoint approach}, which, by modern standards, is a form of ``backpropagation through the operations of a numerical solver''. This approach, widely adopted in the field of finance for calibrating market models~\cite{giles2006smoking}, has high memory cost, and relies on a fixed Euler-Maruyama discretization. 
Recently, this approach was also used by \citet{hegde2019deep} to learn parameterized drift and diffusion functions of an SDE.
In scientific computing, \citet{innes2019zygote} considered backpropagating through high-order implicit SDE solvers.

\citet{ryder2018black} perform variational inference over the state and parameters for Euler-discretized latent SDEs and optimize the model with backpropagation. 
This approach should not be confused with the formulation of variational inference for non-discretized SDEs presented in previous works~\cite{opper2019variational,ha2018adaptive,tzen2019neural} and our work, as it is unclear whether the limit of their discretization corresponds to that obtained by operating with continuous-time SDEs using Girsanov's theorem.

\paragraph{Backward SDEs.} Our stochastic adjoint process relies on the notion of backward SDEs devised by~\citet{kunita2019stochastic}, which is based on two-sided filtrations. 
This is different from the more traditional notion of backward SDEs where only a single filtration is defined~\cite{peng1990general,pardoux1992backward}. 
Based on the latter notion, forward-backward SDEs (FBSDEs) have been proposed to solve stochastic optimal control problems~\cite{peng1999fully}.
However, simulating FBSDEs is costly due to the need to estimate conditional expectations in the backward pass \cite{pardoux1992backward}. 

\paragraph{Bayesian Learning of SDEs.}
Recent works considered the problem of inferring an approximate posterior SDE given observed data under a prior SDE with the same diffusion coefficient~\citep{ha2018adaptive,tzen2019neural,opper2019variational}.
The special case with constant diffusion coefficients was considered more than a decade ago~\cite{archambeau2008variational}.
Notably, computing the KL divergence between two SDEs over a finite time horizon was well-explored in the control literature~\citep{kappen2016adaptive,theodorou2015nonlinear}. 
We include background on this topic in Appendix~\ref{app:latent_sde_bg}.

Bayesian learning and parameter estimation for SDEs have a long history \cite{gupta1974computational}. Techniques which don't require positing a variational family such as the extended Kalman filter and Markov chain Monte Carlo have been considered in the literature \cite{mbalawata2013parameter}.

\vspace{-2mm}
\section{Experiments}
\vspace{-2mm}
The aim of this section is threefold.
We first empirically verify our theory by comparing the gradients obtained by our stochastic adjoint framework against analytically derived gradients for problems having closed-form solutions. 
We then fit latent SDE models with our framework on two synthetic datasets, verifying that the variational inference framework allows learning a generative model of time series. 
Finally, we learn dynamics parameterized by neural networks with a latent SDE from a motion capture dataset, demonstrating competitive performance compared to existing approaches.

We report results based on an implementation of Brownian motion that stores all intermediate queries.
The virtual Brownian tree allowed training with much larger batch sizes on GPUs, but was not necessary for our small-scale experiments.
Notably, our adjoint approach, even when combined with the Brownian motion implementation that stores noise, was able to reduce the memory usage by $1/2$-$1/3$ compared to directly backpropagating through solver operations on the tasks we considered.

\subsection{Numerical Studies}
We consider three test problems (examples 1-3 from ~\cite{rackauckas2017adaptive}; details in Appendix~\ref{app:test_problems}), all of which have closed-form solutions. 
We compare the gradient computed from simulating our stochastic adjoint process using the Milstein scheme against the exact gradient.
Figure~\ref{fig:numerical_study}(a) shows that for test example 2, the error between the adjoint gradient and analytical gradient decreases with step size. 

For all three test problems, the mean squared error across dimensions tends to be smaller as the absolute tolerance of the adaptive solver is reduced (e.g. see Fig. \ref{fig:numerical_study} (b)). 
However, the Number of Function Evaluations (NFEs) tends to be much larger than that in the ODE case~\cite{chen2018neural}.

Additionally, for two out of three test problems, we found that our adjoint approach with the Milstein scheme and fixed step size can be much more time-efficient than regular backpropagation through operations of the Milstein and Euler schemes (see e.g. Fig. \ref{fig:numerical_study}(c)).
Backpropagating through the Euler scheme gives gradients of higher error compared to the Milstein method. 
On the other hand, directly backpropagating through the Milstein solve requires evaluating high-order derivatives and can be costly.

Results for examples 1 and 3 are in Appendix~\ref{app:test_problem_results}.
\begin{figure}[ht]
\begin{minipage}[ht]{\linewidth}
\centering
{\includegraphics[width=0.99\textwidth, clip, trim=14mm 6mm 10mm 6mm]{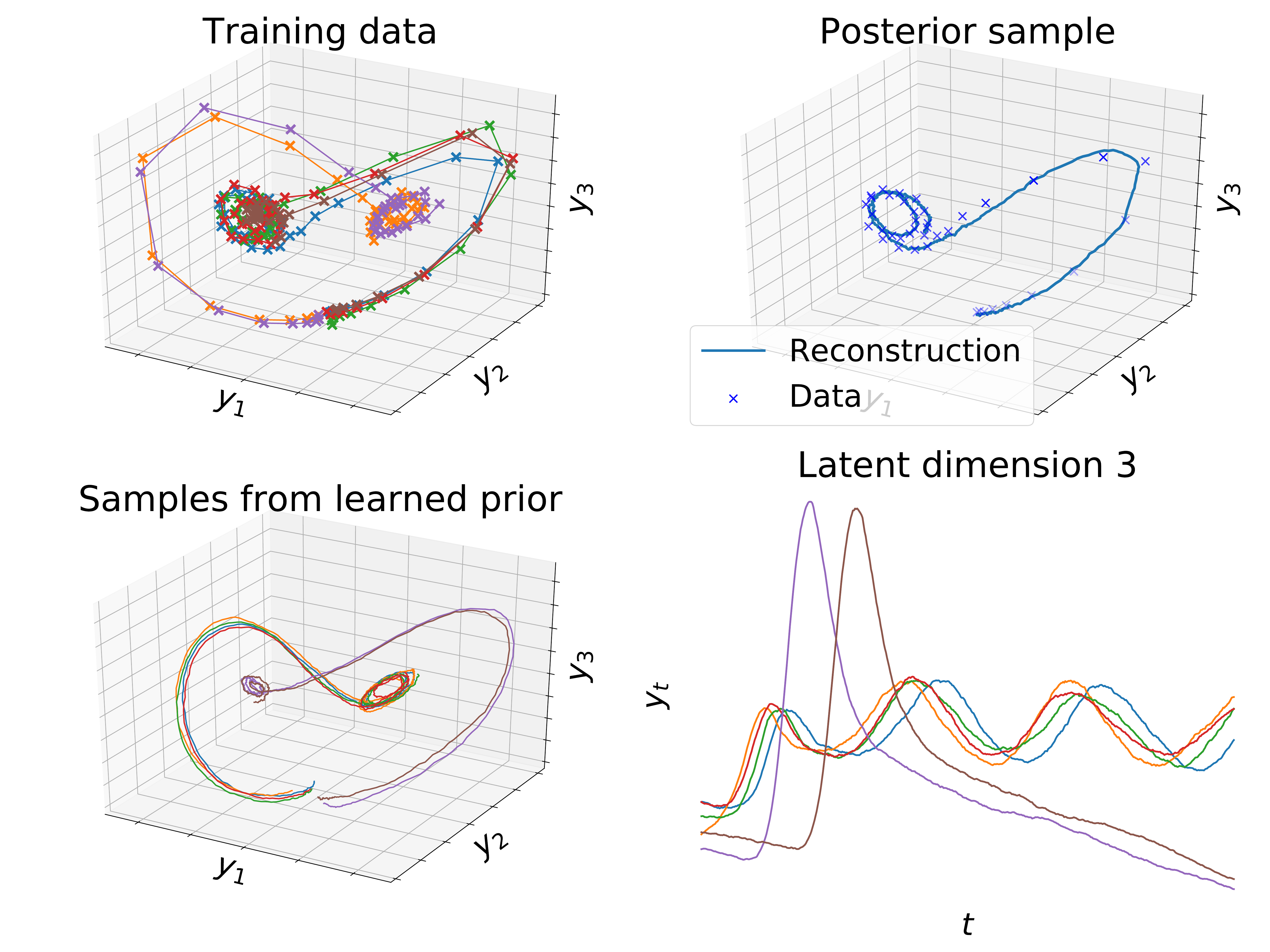}}
\end{minipage}
\caption{
Learned posterior and prior dynamics on data from a stochastic Lorenz attractor.
All samples from our model are continuous-time paths, and form a multi-modal, non-Gaussian distribution.
}
\label{fig:toy_experiments_lorenz}
\end{figure}

\subsection{Synthetic Datasets}
We trained latent SDEs with our adjoint framework to recover (1) a 1D Geometric Brownian motion, and (2) a 3D stochastic Lorenz attractor process. 
The main objective is to verify that the learned posterior can reconstruct the training data, and that the learned priors are not deterministic. 
We jointly optimize the evidence lower bound~(\ref{eq:variational_free_energy}) with respect to parameters of the prior and posterior distributions at the initial latent state $z_0$, the prior and posterior drift, the diffusion function, the encoder, and the decoder.
We include the details of datasets and architectures in Appendix~\ref{app:toy_datasets}.

For the stochastic Lorenz attractor, not only is the model able to reconstruct the data well, but also the learned prior process can produce bimodal samples in both data and latent space. 
This is showcased in the last row of Figure~\ref{fig:toy_experiments_lorenz} where the latent and data space samples cluster around two modes.
This is hard to achieve using a latent ODE with a unimodal Gaussian initial approximate posterior.
We include additional visualizations in Appendix~\ref{app:visualization}.

\newcommand{\rpm}{\raisebox{.2ex}{$\scriptstyle\pm$}}

\subsection{Motion Capture Dataset}
To demonstrate that latent SDEs can learn complex dynamics from real-world datasets, we evaluated their predictive performance on a 50-dimensional motion capture dataset.
The dataset, from~\citet{gan2015deep}, consists of 23 walking sequences of subject 35 partitioned into 16 training, 3 validation, and 4 test sequences. We follow the preprocessing of~\citet{wang2007gaussian}.

In designing the recognition network, we follow~\citet{yildiz2019ode} and use a fully connected network to encode the first three observations of each sequence and thereafter predicted the remaining sequence.
This encoder is chosen for fair comparison to existing models, and could be extended to a recurrent or attention model~\cite{vaswani2017attention}.
The overall architecture is described in Appendix~\ref{app:architecture} and is similar to that of ODE$^2$VAE~\cite{yildiz2019ode}, with a similar number of parameters. 
We also use a fixed step size $1/5$ of smallest interval between any two observations~\cite{yildiz2019ode}.

We train latent ODE and latent SDE models with the Adam optimizer~\cite{kingma2014adam} and its default hyperparameter settings, with an initial learning rate of $0.01$ that is exponentially decayed with rate $0.999$ during each iteration. 
We perform validation over the number of training iterations, KL penalty~\cite{higgins2017beta}, and KL annealing schedule. 
All models were trained for at most $400$ iterations, where we start to observe severe overfitting for most model instances. 
We report the test MSE on future observations following~\citet{yildiz2019ode}. 
We believe that the improved performance is due to the strong regularization in path space, as removing the KL penalty improve training error but caused validation error to deteriorate. 

\begin{table}[h]
\caption{Test MSE on 297 future frames averaged over $50$ samples. $95\%$ confidence interval reported based on t-statistic. $^\dagger$results from~\cite{yildiz2019ode}.
}
\label{tab:mocap_results}
\centering
 \begin{tabular}{l l} 
 \midrule
 Method
 & Test MSE \\
 \midrule
 DTSBN-S~\cite{gan2015deep} & $34.86 \pm 0.02^\dagger$  \\
 {np}ODE~\cite{heinonen2018learning} & $22.96^\dagger$ \\
 {Neural}ODE~\cite{chen2018neural} &  $22.49 \pm 0.88^\dagger$ \\
 $\text{ODE}^2$VAE~\cite{yildiz2019ode} & $10.06 \pm 1.4^\dagger$ \\
 $\text{ODE}^2$VAE-KL~\cite{yildiz2019ode} & $8.09 \pm 1.95^\dagger$\\
 Latent ODE~\cite{chen2018neural,rubanova2019latent} & $5.98 \pm 0.28$ \\
 Latent SDE (this work) & $\mathbf{4.03 \pm 0.20}$ \\
 \midrule
 \end{tabular}
\end{table}

\section{Discussion}
We presented a generalization of the adjoint sensitivity method to compute gradients through solutions of SDEs. 
In contrast to existing approaches, this method has nearly the same time and memory complexity as simply solving the SDE.
We showed how our stochastic adjoint framework can be combined with a gradient-based stochastic variational inference scheme for training latent SDEs.

It is worthwhile to mention that SDEs and the commonly used GP models define two distinct classes of stochastic processes, albeit having a nonempty intersection (e.g. Ornstein-Uhlenbeck processes fall under both).
Computationally, the cost of fitting GPs lies in the matrix inversion, whereas the computational bottleneck of training SDEs is the sequential numerical solve. 
Empirically, another avenue of research is to reduce the variance of gradient estimates.
In the future, we may adopt techniques such as control variates or antithetic paths.

On the application side, our method opens up a broad set of opportunities for fitting any differentiable SDE model, such as Wright-Fisher models with selection and mutation parameters~\cite{ewens2012mathematical}, derivative pricing models in finance, or infinitely-deep Bayesian neural networks~\cite{neuralSDEBNN}.
In addition, the latent SDE model enabled by our framework can be extended to include domain knowledge and structural or stationarity constraints~\cite{ma2015complete} in the prior process for specific applications. 

On the theory side, there remain fundamental questions to be answered. 
Convergence rates of numerical gradients estimated with general schemes are unknown.
Additionally, since our analyses are based on strong orders of schemes, it is natural to question whether convergence results still hold when we consider weak errors, and moreover if the method could be reformulated more coherently with rough paths theory \cite{lyons1998differential}.

\subsubsection*{Acknowledgements}
We thank Yulia Rubanova, Danijar Hafner, Mufan Li, Shengyang Sun, Kenneth R. Jackson, Simo S\"arkk\"a, Daniel Lacker, and Philippe Casgrain for helpful discussions. 
We thank Çağatay Yıldız for helpful discussions regarding evaluation settings of the mocap task. 
We also thank Guodong Zhang, Kevin Swersky, Chris Rackauckas, and members of the Vector Institute for helpful comments on an early draft of this paper.

\bibliographystyle{plainnat}
\bibliography{main}

\newpage
\onecolumn
\section{Appendix}

\subsection{Notation} \label{app:notation}
For a fixed terminal time $T > 0$, we denote by $\mathbb{T} = [0, T] \subseteq \R$ the time horizon. Let 
$C^{\infty}$ be the class of infinitely differentiable functions from $\R^d$ to itself.
Let $C^{p, q}$ be the class of functions from $\R^d \times \mathbb{T}$ to $\R^d$ that are $p$ and $q$ times continuously differentiable in the first and second input, respectively. 
Let $C_b^{p, q} \subseteq C^{p, q}$ be the subclass with bounded derivatives of all possible orders.
For a positive integer $m$, we adopt the shorthand $[m] = \{1, 2, \dots, m\}$. 
We denote the Euclidean norm of a vector $v$ by $| v |$.
For $f \in C^{p, q}$, we denote its Jacobian with respect to the first input by $\nabla f$.
We denote the concatenation of two vectors $u \in \R^{d_1}$ and $v \in \R^{d_2}$ by the simplified notation $(u, v)$, as opposed to the slightly lengthy notation $(u^\top, v^\top)^\top$.

\subsection{Proof of Theorem \ref{lemm:jacobian_dynamics}}
\begin{proof}[Proof of Theorem \ref{lemm:jacobian_dynamics}]
\label{app:jacobian_dynamics}
We have $J_{s, t}(z) = \nabla \widecheck{\Psi}_{s, t}(z)$, where $\widecheck{\Psi}_{s, t}(z)$ is defined in (\ref{eq:backward_stratonovich_sde}).
Now we take the gradient with respect to $z$ on both sides. The solution is differentiable with respect to $z$ and we may differentiate under the stochastic integral~\cite[Proposition 2.4.3]{kunita2019stochastic}. 
Theorem 3.4.3~\cite{kunita2019stochastic} is sufficient for the regularity conditions required.
Since $K_{s, t}(z) = J_{s, t}(z)^{-1}$, applying the Stratonovich version of It\^{o}'s formula to (\ref{eq:dynamics.J}), we have (\ref{eq:dynamics.K}).
\end{proof}

\subsection{Proof of Theorem~\ref{thm:approx.scheme}}\label{app:approx.scheme}

\begin{proof} [Proof of Theorem \ref{thm:approx.scheme}]
By the triangle inequality,
\begin{equation}
\begin{split}
&|\mathsf{F}(\mathsf{G}(z, W_{\cdot}), W_{\cdot}) - \mathsf{F}_h(\mathsf{G}_h(z, W_{\cdot}), W_{\cdot})| \\
\leq &
	\underbrace{
		|\mathsf{F}(\mathsf{G}(z, W_{\cdot}), W_{\cdot}) - \mathsf{F}(\mathsf{G}_h(z, W_{\cdot}), W_{\cdot})|
	}_{I_h^{(1)}} +
	\underbrace{
		|\mathsf{F}(\mathsf{G}_h(z, W_{\cdot}), W_{\cdot}) - \mathsf{F}_h(\mathsf{G}_h(z, W_{\cdot}), W_{\cdot})|
	}_{I_h^{(2)}}.
\end{split}
\end{equation}
We show that both $I_h^{(1)}$ and $I_h^{(2)}$ converge to $0$ in probability as $h \to 0$.
For simplicity, we suppress $z$ and $W_{\cdot}$.

{\it Bounding $I_h^{(1)}$.} 
Let $\epsilon > 0$ be given. 
Since $G_h \rightarrow G$ in probability, there exist $M_1 > 0$ and $h_0 > 0$ such that
\[
\mathbb{P}(|G| > M_1) < \epsilon, \quad \mathbb{P}(|G_h| > 2M_1) < \epsilon, \quad \text{for all } h \leq h_0.
\]
By Lemma 2.1 (iv) of \citet{ocone1989generalized}, which can be easily adapted to our context, there exists a positive random variable $C_1$, finite almost surely, such that $\sup_{|z| \leq 2 M_1} \left| \nabla_z \mathsf{F} \right| \leq C_1$, and there exists $M_2 > 0$ such that $\mathbb{P}( |C_1| > M_2 ) < \epsilon$. Given $M_2$, there exists $h_1 > 0$ such that
\[
\mathbb{P}\bracks{ |G - G_h| > \frac{\epsilon}{M_2} } < \epsilon, \quad \text{for all } h \leq h_1.
\]
Now, suppose $h \leq \min\{h_0, h_1\}$. Then, by the union bound, with probability at least $1 - 4 \epsilon$, we have
\[
|G| \leq M_1, \quad  |G_h| \leq 2M_1, \quad |C_1| \leq M_2, \quad |G - G_h| \leq \frac{\epsilon}{M_2}.
\]	
On this event, we have
\[
I_h^{(1)} = |\mathsf{F}(\mathsf{G}) - \mathsf{F}(\mathsf{G}_h)| \leq C_1 |G - G_h| \leq M_2 \frac{\epsilon}{M_2} = \epsilon.
\]
Thus, we have shown that $I_h^{(1)}$ converges to $0$ in probability as $h \rightarrow 0$.

\medskip

{\it Bounding $I_h^{(2)}$.} The idea is similar. By condition (ii), we have
\[
\lim_{h \rightarrow 0} \sup_{|z_T|\leq M} |\mathsf{F}_h(z_T) - \mathsf{F}(z_T)| = 0
\]
in probability. 
Using this and condition (i), for given $\epsilon > 0$, there exist $M > 0$ and $h_2 > 0$ such that for all $h \le h_2$, we have
\[
|G_h| \leq M \quad \text{and} \quad \sup_{|z_T|\leq M} |\mathsf{F}_h(z_T) - \mathsf{F}(z_T)| < \epsilon
\]
with probability at least $1 - \epsilon$. On this event, we have
\[
|\mathsf{F}(\mathsf{G}_h) - \mathsf{F}_h(\mathsf{G}_h)| \leq \sup_{|z_T|\leq M} |\mathsf{F}_h(z_T) - \mathsf{F}(z_T)| < \epsilon.
\]
Thus $I_h^{(2)}$ also converges to $0$ in probability as $h\to 0$.
\end{proof}

\subsection{Euler-Maruyama Scheme Satisfies Local Uniform Convergence}\label{app:euler_cond2}

Here we verify that the Euler-Maruyama scheme satisfies condition $(ii)$ when $d=1$. 
Our proof can be extended to the case where $d > 1$ assuming an $L^p$ estimate of the error; see the discussion after the proof of Proposition \ref{prop:euler_cond2}.

\begin{prop}\label{prop:euler_cond2}
Let $\mathsf{F}_h(z)$ be the Euler-Maruyama discretization of a $1$-dimensional SDE with mesh size $h$ of $\mathsf{F}(z)$. Then, for any compact $A \subset \mathbb{R}$, we have
\eq{ \label{eqn:convergence.result.claim}
    \plim_{h\to0} \sup_{z \in A} | \mathsf{F}_h(z) - \mathsf{F}(z) | = 0.
}
\end{prop}

Usual convergence results in stochastic numerics only control the error for a single fixed starting point. Here, we strengthen the result to local uniform convergence. Our main idea is to apply a Sobolev inequality argument \cite[Part II]{ocone1989generalized}. To do so, we need some preliminary results about the Euler-Maruyama discretization of the original SDE and its derivative.
We first recall a theorem characterizing the expected squared error for general schemes.

\begin{theo}[Mean-square order of convergence {\cite[Theorem 1.1]{milstein2013stochastic}}]\label{theo:fundamental_convergence}
Let $\{Z_t^z\}_{t\ge 0}$ be the solution to an It\^o SDE, and $\{\tilde{Z}_k^z\}_{k\in \mathbb{N}}$ be a numerical discretization with fixed step size $h$, both of which are started at $z \in \R^d$ and defined on the same probability space.
Let the coefficients of the SDE be $C_b^{1, \infty}$. Furthermore, suppose that the numerical scheme has order of accuracy $p_1$ for the expectation of deviation and order of accuracy $p_2$ for the mean-square deviation. If $p_1 \ge p_2 + 1/2$ and $p_2 \ge 1/2$, then, for any $N \in \mathbb{N}$, $k \in [N]$, and $z \in \R^d$
\eq{ \label{eqn:known.estimate}
    \Exp{
        |Z_{t_k}^z - \tilde{Z}_k^z |^2
    } \le
        C
        \bracks{
            1 + |z|^2
        }
        h^{2p_2 - 1},
}
for a constant $C$ that does not depend on $h$ or $z$.
\end{theo}

We refer the reader to \citep{milstein2013stochastic} for the precise definitions of orders of accuracy and the proof. Given this theorem, we establish an estimate regarding errors of the discretization and its derivative with respect to the initial position.

\begin{lemm}\label{lemm:estimate}
We have
\eq{
    \Exp{
        | \mathsf{F}(z) - \mathsf{F_h}(z)|^2 + 
        | \nabla_z \mathsf{F}(z) - \nabla_z \mathsf{F_h}(z)|^2
    }
    \le&
    C_1(1 + |z|^2) h, \label{eq:f_h}
}
where $C_1$ is a constant independent of $z$ and $h$.
\end{lemm}

\begin{proof}[Proof of Lemma \ref{lemm:estimate}]
Since the coefficients of the SDE are of class $C_b^{\infty, 1}$, we may differentiate the SDE in $z$ to get the SDE for the derivative $\nabla_z Z^z_t$ \citep{kunita2019stochastic}. 
Specifically, letting $Y_t^z = \nabla_z Z_t^z$, we have
\eq{
    Y_t^z =
        I_d +
        \int_0^t \nabla b(Z_s^z , s) Y_s^z \ds + 
        \int_0^t \nabla \sigma(Z_s^z, s) Y_s^z \dW_s.
}

Note that the augmented process $(\mathsf{F}(z), \nabla_z \mathsf{F}(z))$ satisfies an SDE with $C_b^{\infty, 1}$ coefficients. 
By the chain rule, one can easily show that the derivative of the Euler-Maruyama discretization $\mathsf{F}_h(z)$ is the discretization of the derivative process $Y_t^z$. 
Thus, $(\mathsf{F}_h(z), \nabla_z \mathsf{F}_h(z))$ is simply the discretization of $(\mathsf{F}(z), \nabla_z \mathsf{F}(z))$.

Since the Euler-Maruyama scheme has orders of accuracy $(p_1, p_2) = (1.5, 1.0)$ \citep[Section 1.1.5]{milstein2013stochastic}, by Theorem \ref{theo:fundamental_convergence}, we have 
\eq{
    \Exp{
        |\mathsf{F}(z) - \mathsf{F}_h(z)|^2 + 
        |\nabla_z \mathsf{F}(z) - \nabla_z \mathsf{F}_h(z)|^2
    }
    \le
    C_1 (1 + |z|^2 ) h, \quad z \in \mathbb{R}^d
}
for some constant $C_1$ that does not depend on $z$ or $h$.
\end{proof}

We also recall a variant of the Sobolev inequality which we will apply for $d = 1$.

\begin{theo}[Sobolev inequality {\cite[Theorem 5.4.1.c]{adams1975sobolev}}]\label{theo:Sobolev}
For any $p > d$, there exists a universal constant $c_p$ such that 
\eq{
    \sup_{x \in \R^d} |f(x)| \le c_p \norm{f}_{1, p},
}
where 
\eq{
    \norm{f}_{1,p}^p := 
        \int_{\R^d} | f(x) |^p \dx + 
        \int_{\R^d} | \nabla_x f(x) |^p \dx,
}
for all continuously differentiable $f: \R^d \to \R$.
\end{theo}

\begin{proof}[Proof of Proposition \ref{prop:euler_cond2}]
Define $\mathsf{H}^\alpha_h: \Omega \times \R \to \R$, regarded as a random function $\mathsf{H}^\alpha_h(\omega): \mathbb{R} \rightarrow \mathbb{R}$, by 
\eq{
    \mathsf{H}_h^\alpha(z) =
        \frac{\mathsf{F}(z) - \mathsf{F}_h(z)}{
            (1 + |z|^2)^{1/2 + \alpha}
        },
}
where $\alpha > 1/2$ is a fixed constant. Since $\mathsf{H}_h^\alpha$ is continuously differentiable a.s., by Theorem \ref{theo:Sobolev},
\eq{
    | \mathsf{F}(z) - \mathsf{F}_h(z) | 
    \le 
        c_2 (1 + |z|^2 )^{1/2 + \alpha} \norm{\mathsf{H}_h^\alpha}_{1,2},
    \quad \text{for all} \; z \in \R \quad a.s.
}
Without loss of generality, we may let the compact set be $A = \{z : |z| \leq M\}$ where $M > 0$. Then,
\eq{  \label{eqn:bound.from.Sobolev}
    \sup_{|z| \le M}
        | \mathsf{F}(z) - \mathsf{F}_h(z) | 
    \le 
        c_2 (1 + M^2 )^{1/2 + \alpha} \norm{\mathsf{H}_h^\alpha}_{1,2},
    \quad a.s.
}

It remains to estimate $\norm{\mathsf{H}_h^\alpha}_{1,2}$. Starting from the definition of $\norm{\cdot}_{1, p}$, a standard estimation yields
\eq{
    \norm{\mathsf{H}_h^\alpha}_{1,2}^2
    \le&
    C_2 \int_{\mathbb{R}} \frac{
        |\mathsf{F}(z) - \mathsf{F}_h(z)|^2 + |\nabla_z\mathsf{F}(z) - \nabla_z\mathsf{F}_h(z)|^2
    }{
        (1 + |z|^2)^{1 + 2\alpha}
    } \dz,
}
where $C_2$ is a deterministic constant depending only on $\alpha$ (but not $z$ and $h$). 

Now we take expectation on both sides. By Lemma \ref{lemm:estimate}, we have
\eq{
    \Exp{
        \norm{\mathsf{H}_h^\alpha}_{1,2}^2
    }
    \le& C_2 \int_{\mathbb{R}} \frac{\mathbb{E} [|\mathsf{F}(z) - \mathsf{F}_h(z)|^2 + |\nabla_z\mathsf{F}(z) - \nabla_z\mathsf{F}_h(z)|^2] }{(1 + |z|^2)^{1 + 2\alpha}} \dz,\\
    \le& C_1 C_2 h \int_{\mathbb{R}} \frac{1}{(1 + |z|^2)^{2\alpha}} \dz,
}
where the last integral is finite since $\alpha > 1/2$.

We have shown that $\Exp{\norm{\mathsf{H}_h^\alpha}_{1,2}^2} = \mathcal{O}(h)$. Thus $\norm{\mathsf{H}_h^\alpha}_{1,2} \to 0$ in $L^2$, and hence also in probability, as $h\to 0$. From \eqref{eqn:bound.from.Sobolev}, we have that $\sup_{z \in A} | \mathsf{F}_h(z) - \mathsf{F}(z) |$ converges to $0$ in probability as $h \rightarrow 0$.
\end{proof}

It is clear from the above proof that we may generalize to the case where $d > 1$ and other numerical schemes if we can bound the expected $W^{1, p}$-norm of $\mathsf{F}_h - \mathsf{F}$ in terms of $z$ and $h$, for $p > d$,
where $W^{1, p}$ here denotes the Sobolev space consisting of all real-valued functions on $\R^d$ whose weak derivatives are functions in $L^p$.
For the Euler scheme and $d > 1$, we need only bound the $L^p$ norm of the discretization error in terms of $z$ and $h$ for general $p$.
To achieve this, we would need to make explicit the dependence on $z$ for existing estimates (see e.g. \cite[Chapter 10]{kloeden2013numerical}).

Generically extending the argument to other numerical schemes, however, is technically non-trivial.
We plan to address this question in future research.

\subsection{Stochastic Adjoint has Commutative Noise when Original SDE has Diagonal Noise} \label{app:commutativity}
\newcommand{\aug}{ {\text{aug}} }
Recall the Stratonovich SDE~(\ref{eq:stratonovich_sde}) with drift and diffusion functions $b, \sigma_1, \dots, \sigma_m \in \R^d \times \R \to \R^d$ being governed by a set of parameters $\theta \in \R^p$.
Consider the augmented state composed of the state and parameters $Y_t = (Z_t, \theta)$. 
The augmented state satisfies a Stratonovich SDE with the drift function $f(y, t) = ( b(z, t), \vzero_p)$ and diffusion functions $g_i(y, t) = (\sigma_i(z, t), \vzero_p)$ for $i \in [m]$. 
We have omitted the dependence on parameter $\theta$ to reduce notational clutter. 
By (\ref{eq:dynamics.K}) and (\ref{eqn:A.Atilde}), the adjoint process of the augmented state follows the backward Stratonovich SDE:
\eq{
    \widecheck{A}^y_t = 
        \widecheck{A}^y_T + 
        \int_t^T \widecheck{A}_s^y \nabla f(\widecheck{Y}_s, s) \ds + 
        \sum_{i=1}^m \int_t^T \widecheck{A}_s^y \nabla g_i(\widecheck{Y}_s, s) \circ \dee \widecheck{W}_s^{(i)}.
}
By definitions of $f$ and $g_i$, the Jacobian matrices $\nabla f(x, s)$ and $\nabla g_i(x, s)$ have the forms
\eq{
\label{eq:Jacobian_fg}
    \nabla f(y, s) = \begin{pmatrix}
        \nabla b(z, s)    & \vzero_{d \times p} \\
        \vzero_{p \times d} & \vzero_{p \times p}
    \end{pmatrix}\in \R^{(d + p) \times (d + p)}, \quad
    \nabla g_i(y, s) = \begin{pmatrix}
        \nabla \sigma_i(z, s) & \vzero_{d \times p}\\
        \vzero_{p \times d}     & \vzero_{p \times p}
    \end{pmatrix}\in \R^{(d + p) \times (d + p)}.
}
\vspace{-1mm}
Thus, the backward Stratonovich SDEs for the adjoint processes of the state and parameters have the forms
\eq{
\label{eq:adjoint_state_param}
    \widecheck{A}^z_t =& 
        \widecheck{A}^z_T +
        \int_t^T
            \widecheck{A}_s^z \frac{\partial b(z, s)}{\partial z}\bigg|_{z=\widecheck{Z}_s} \ds +
        \sum_{i=1}^m \int_t^T
            \widecheck{A}_s^z \frac{\partial \sigma_i(z, s)}{\partial z}\bigg|_{z=\widecheck{Z}_s} \circ \dee \widecheck{W}_s^{(i)}, \\
    \widecheck{A}^\theta_t =& 
        \widecheck{A}^\theta_T + 
        \int_t^T 
            \widecheck{A}_s^z \frac{\partial b(z, s)}{\partial \theta}\bigg|_{z=\widecheck{Z}_s} \ds +
        \sum_{i=1}^m \int_t^T 
            \widecheck{A}_s^z \frac{\partial \sigma_i(z, s)}{\partial \theta}\bigg|_{z=\widecheck{Z}_s} \circ \dee \widecheck{W}_s^{(i)}.
}
Now assume the original SDE has diagonal noise. Then, $m=d$ and Jacobian matrix $\nabla \sigma_i(z)$ has the form
\eq{
\label{eq:Jacobian_sigma}
    \nabla \sigma_i(z) = 
    \mypmatrix{
        0 & ... & 0 & 0 & 0 & ... & 0 \\
        0 & ... & 0 & \frac{\partial \sigma_{i,i}(z)}{\partial z_i} & 0 & ... & 0 \\
        0 & ... & 0 & 0 & 0 & ... & 0
    }.
}
Consider the adjoint process for the augmented state along with the backward flow of the backward Stratonovich SDE~(\ref{eq:backward_stratonovich_sde}), whose overall state we denote by $\widecheck{X}_t = ( \widecheck{Z}_t, \widecheck{A}_t^z, \widecheck{A}_t^\theta )$. 
By (\ref{eq:adjoint_state_param}) and (\ref{eq:Jacobian_sigma}), $\{\widecheck{X}_t\}_{t \in \mathbb{T}}$ satisfies a backward Stratonovich SDE with a diffusion of the form
\eq{
\label{eq:huge_diffusion_fn}
    G(x) =
    \bracks{
    \begin{array}{c;{2pt/2pt}c c c; {2pt/2pt}c}
        -\sigma_{1,1}(z_1) & 0 & \dots & 0 & 0 \\
        & & \ddots & & \\
        0 & 0 & \dots & 0 & -\sigma_{d,d}(z_d) \\ \hdashline
        \frac{\partial \sigma_{1,1}(z_1)}{\partial z_1} a^z_1 & 0 & \dots & 0 & 0 \\
        & & \ddots & & \\
        0 & 0 & \dots & 0 & \frac{\partial \sigma_{d,d}(z_d)}{\partial z_d} a^z_d \\ \hdashline
        \frac{\partial \sigma_{1,1}(z_1)}{\partial \theta_1} a^z_1 & \dots & \dots & \dots & \frac{\partial \sigma_{d,d}(z_d)}{\partial \theta_1} a^z_d \\
        \dots & \dots & \dots & \dots & \dots \\
        \frac{\partial \sigma_{1,1}(z_1)}{\partial \theta_p} a^z_1 & \dots & \dots & \dots & \frac{\partial \sigma_{d,d}(z_d)}{\partial \theta_p} a^z_d
    \end{array}
    }
    \in \R^{(2d + p) \times d},
}
where $x = (z, a^z, a^\theta)$, and the subscript indexes the dimension. 
Recall, for an SDE with diffusion function $\Sigma(x) \in \R^{d \times m}$, it is said to satisfy the commutativity property~\cite{rossler2004runge} if 
\eq{
\label{eq:commutativity}
    \sum_{i=1}^d \Sigma_{i, j_2} (x) \frac{\partial \Sigma_{k, j_1} (x)}{\partial x_i} = 
    \sum_{i=1}^d \Sigma_{i, j_1} (x) \frac{\partial \Sigma_{k, j_2} (x)}{\partial x_i},
}
for all $j_1, j_2 \in [m]$ and $k \in [d]$.
When an SDE has commutative noise, the computationally intensive double It\^o integrals (and the L\'evy areas) need not be simulated by having the numerical scheme take advantage of the following property of iterated integrals~\cite{ilie2015adaptive}:
\eq{
    \int_s^t \int_s^u \dW_r^{(i)} \dW_u^{(j)} + \int_s^t \int_s^u \dW_r^{(j)} \dW_u^{(i)}
    = \Delta W^{(i)} \Delta W^{(j)},
}
where the Brownian motion increment $\Delta W^{(i)} = W_t^{(i)} - W_s^{(i)}$ for $i \in [m]$ can be easily sampled.

We show the diffusion function (\ref{eq:huge_diffusion_fn}) satisfies the commutativity condition (\ref{eq:commutativity}) with a proof by exhaustion:
\begin{description}
    \item[Case 1: $k = 1, \dots, d$.] 
        Both LHS and RHS are zero unless $j_1 = j_2 = k$, since for $\Sigma_{i, j_2} (x) \frac{\partial \Sigma_{k, j_1} (x)}{\partial x_i}$ to be non-zero, $i = j_1 = j_2 = k$.
    \item[Case 2: $k = d+1 \dots, 2d$.]
        Similar to the case above. 
    \item[Case 3: $k = 2d+1 \dots, 2d + p$.] 
        Write $k = 2d + l$, where $l \in [p]$. Both LHS and RHS are zero unless $j_1 = j_2 = l$, since for $\Sigma_{i, j_2} (x) \frac{\partial \Sigma_{k, j_1} (x)}{\partial x_i}$ to be non-zero $i = l$ or $i= d + l$ and $j_1 = j_2 = l$.
\end{description}
This concludes that the commutativity condition holds.
Finally, we comment that the Milstein scheme for the stochastic adjoint of diagonal noise SDEs can be implemented such that during each iteration of the backward solve, \texttt{vjp} is only called a number of times independent of the dimensionality of the original SDE. 

\subsection{Background on Latent SDE} \label{app:latent_sde_bg}
Consider a filtered probability space $(\Omega, \F, \{\F_t\}_{0 \leq t \leq T}, P)$, where $\mathbb{T} = [0, T]$ is a finite time horizon.

Recall the approximate posterior process that we intend to learn is governed by the SDE:
\eq{
\label{eq:before_girsanov}
\dZ_t = h_{\phi}(Z_t, t) \dt + \sigma(Z_t, t) \dW_t, \quad Z_0 = z_0 \in \R^d.
}

Suppose there exists a measurable function $u(z, t)$ such that 
\begin{itemize}
    \item $\sigma(z, t) u(z, t) = h_\phi(z, t) - h_\theta(z, t)$, and
    \item $u(Z_t, t)$ satisfies Novikov's condition, i.e. 
            $\Exp{ \exp \bracks{ \int_0^T \frac{1}{2} |u(Z_t, t)|^2 \dt } } < \infty$.
\end{itemize}
Novikov's condition ensures that the process
\eq{
	M_t = \exp \bracks{
		- \int_0^t \frac{1}{2} |u(Z_s, s)|^2 \ds
		- \int_0^t u(Z_s, s)^\top \dW_s
	}, \quad 0 \le t \le T,
}
is a $P$-martingale. By Girsanov Theorem II~\cite[Theorem 8.6.4]{oksendal2013stochastic}, the process
$\widehat{W}_t = \int_0^t u(Z_s, s) \ds + W_t$, $0 \leq t \leq T$ is a Wiener process under the probability measure $Q$ defined by 
\eq{
\dee Q = M_T \dee P,       %
}
Moreover, since a simple rewrite shows that 
\eq{
\label{eq:after_girsanov}
    \dZ_t = h_\theta(Z_t, t)\dt + \sigma(Z_t, t) \dee \widehat{W}_t, \quad Z_0 = z_0,
}
we conclude that the $Q$-law of (\ref{eq:after_girsanov}) (or equivalently (\ref{eq:before_girsanov})) is the same as the $P$-law of the prior process. 

\subsubsection{Deriving the Variational Bound}
Let $x_{t_1}, \dots, x_{t_N}$ be observed data at times $t_1, \dots, t_N$, whose conditionals only depend on the respective latent states $z_{t_1}, \dots, z_{t_N}$. 
Since the $Q$-law of the approximate posterior is the same as the $P$-law of the prior,
\eq{
    \log p(x_{t_1}, \dots, x_{t_N} ) =
    \log \E_P \left [
        \prod_{i=1}^N p(x_{t_i} | \tilde{z}_{t_i} )
    \right]
    =&\log \E_Q \left[
        \prod_{i=1}^N p(x_{t_i} | z_{t_i})
    \right] \\
    =& 
    \log \E_P \left[
        \prod_{i=1}^N p(x_{t_i} | z_{t_i}) M_T
    \right] \\
    \ge&
    \E_P \left[
        \sum_{i=1}^N \log p(x_{t_i} | z_{t_i}) + \log M_T
    \right] \\
    =&
    \E_P \left[
        \sum_{i=1}^N \log p(x_{t_i} | z_{t_i}) 
        - \int_0^T \frac{1}{2} |u(Z_t, t)|^2 \dt
        - \int_0^T u(Z_t)^\top \dW_t
    \right]\\
    =&
    \E_P \left[
        \sum_{i=1}^N \log p(x_{t_i} | z_{t_i}) 
        - \int_0^T \frac{1}{2} |u(Z_t, t)|^2 \dt
    \right],
}
where the second line follows from the definition of $Q$ and third line follows from Jensen's inequality. In the last equality we used the fact that the It\^{o} integral $\int_0^{\cdot} u(Z_t)^{\top} dW_t$ is a martingale.

\subsection{Stochastic Adjoint for Latent SDE}\label{app:latent_sde_adjoint}
To simulate the variational lower bound (\ref{eq:variational_free_energy}) with Monte Carlo in the forward pass along with the original dynamics, we need only extend the original augmented state with an extra variable $L_t$ such that the new drift and diffusion functions for the new augmented state $Y_t = (Z_t, \theta, L_t)$ are
\eq{
    f(x, t) = \begin{pmatrix}
    b(z, t) \\
    \vzero_p \\
    \tfrac{1}{2} | u(z, t) |_2^2
    \end{pmatrix} \in \R^{d + p + 1}
    , \quad
    g_i(x, t) = \begin{pmatrix}
    \sigma_i(z, t)\\
    \vzero_p\\
    0
    \end{pmatrix}\in \R^{d + p + 1} 
    , \quad i \in [m].
}
By (\ref{eqn:A.tilde.dynamics}), the backward SDEs of the adjoint processes become
\begin{equation}
\begin{aligned}
\label{eq:latent_sde_adjoint_state_param}
    \widecheck{A}^z_t =&
        \widecheck{A}^z_T +
        \int_t^T \bracks{
            \widecheck{A}_s^z \frac{\partial b(z, s)}{\partial z}\bigg|_{z=\widecheck{Z}_s} + 
            \frac{1}{2} \widecheck{A}_s^l \frac{\partial |{ u(z, s) }|_2^2 }{\partial z}\bigg|_{z=\widecheck{Z}_s}  } \ds +
        \sum_{i=1}^m \int_t^T
            \widecheck{A}_s^z \frac{\partial \sigma_i(z, s)}{\partial z}\bigg|_{z=\widecheck{Z}_s} \circ \dee \widecheck{W}_s^{(i)}, \\
    \widecheck{A}^\theta_t =& 
        \widecheck{A}^\theta_T +
        \int_t^T \bracks{
            \widecheck{A}_s^z \frac{\partial b(z, s)}{\partial \theta}\bigg|_{z=\widecheck{Z}_s} + 
            \frac{1}{2} \widecheck{A}_s^l \frac{\partial |{ u(z, s) }|_2^2 }{\partial \theta}\bigg|_{z=\widecheck{Z}_s} } \ds +
        \sum_{i=1}^m \int_t^T
            \widecheck{A}_s^z \frac{\partial \sigma_i(z, s)}{\partial \theta}\bigg|_{z=\widecheck{Z}_s} \circ \dee \widecheck{W}_s^{(i)}, \\
    \widecheck{A}^l_t =& \widecheck{A}^l_T.
\end{aligned}
\end{equation}
In this case, neither does one need to simulate the backward SDE of the extra variable nor does one need to simulate its adjoint. Moreover, when considered as a single system for the augmented adjoint state, the diffusion function of the backward SDE (\ref{eq:latent_sde_adjoint_state_param}) satisfies the commutativity property (\ref{eq:commutativity}).

\subsection{Test Problems}\label{app:test_problems}
In the following, $\alpha, \beta$, and $p$ are parameters of SDEs, and $x_0$ is a fixed initial value.

\paragraph{Example 1.}
\eq{
\dX_{t}=\alpha X_{t} \dt+\beta X_{t} \dW_{t}, \quad X_0 = x_0.
}
Analytical solution:
\eq{
X_{t}=X_{0} e^{\left(\beta-\frac{\alpha^{2}}{2}\right) t+\alpha W_{t}}.
}

\paragraph{Example 2.}
\eq{
\dX_{t}=&
-\left(p^2\right)^{2} \sin \left(X_{t}\right) \cos ^{3}\left(X_{t}\right) \dt +p \cos ^{2}\left(X_{t}\right) \dW_{t}, \quad X_0 = x_0.
}
Analytical solution:
\eq{
X_{t}=\arctan \left(p W_{t}+\tan \left(X_{0}\right)\right).
}

\paragraph{Example 3.}
\eq{
\dX_{t}=&
    \left(\frac{\beta}{\sqrt{1+t}}-\frac{1}{2(1+t)} X_{t}\right) \dt+
    \frac{\alpha \beta}{\sqrt{1+t}} \dW_{t}, \quad X_0 = x_0.
}
Analytical solution:
\eq{
X_{t}=\frac{1}{\sqrt{1+t}} X_{0}+\frac{\beta}{\sqrt{1+t}}\left(t+\alpha W_{t}\right).
}
In each numerical experiment, we duplicate the equation 10 times to obtain a system of SDEs where each dimension had their own parameter values sampled from the standard Gaussian distribution and then passed through a sigmoid to ensure positivity. Moreover, we also sample the initial value for each dimension from a Gaussian distribution. 
\subsection{Results for Example 1 and 3}\label{app:test_problem_results}

\begin{figure*}[ht]\label{fig:additional_numerical_study}
\begin{minipage}[t]{0.333\linewidth}
\centering
{\includegraphics[width=0.98\textwidth]{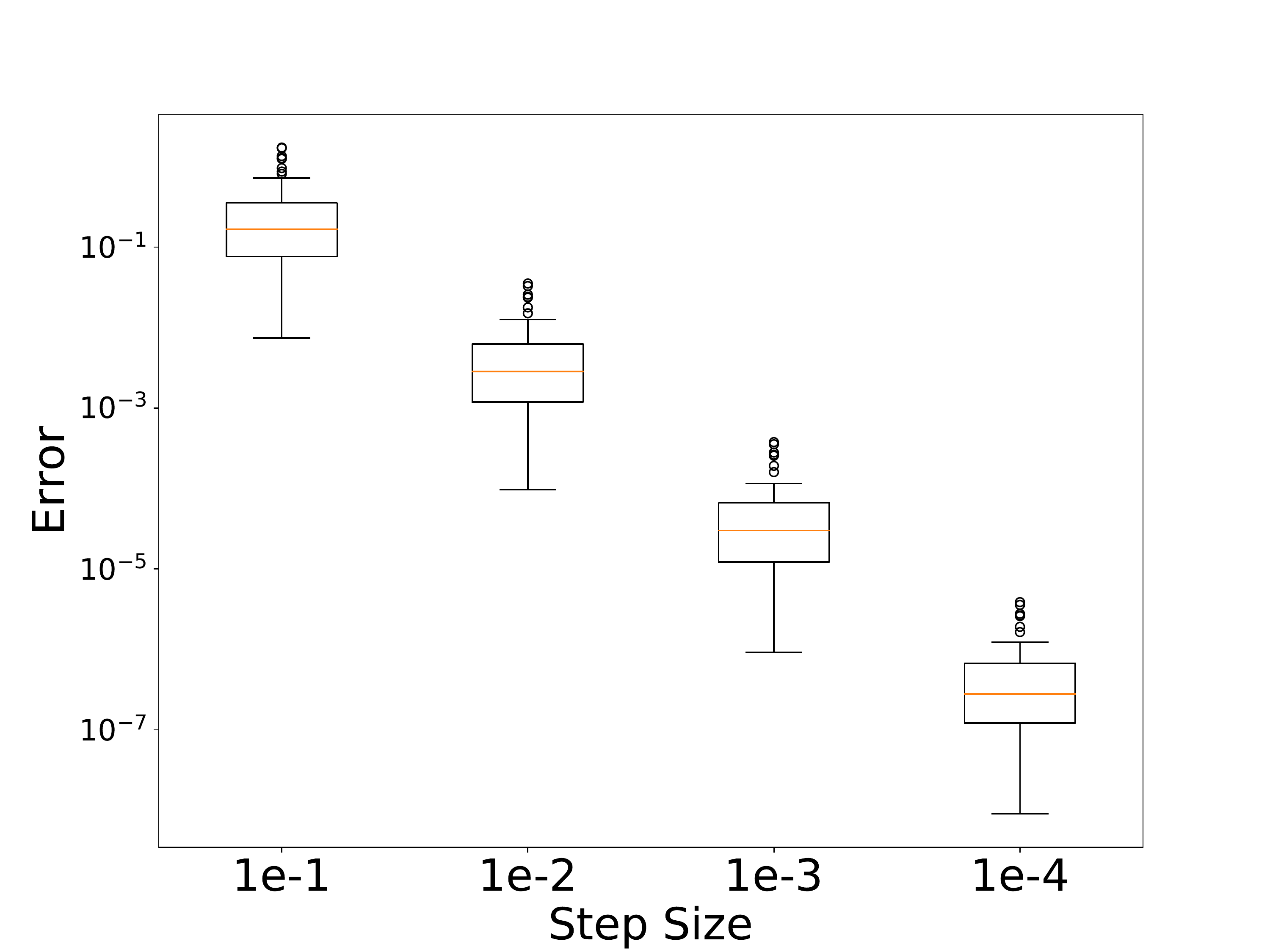}} \\ \vspace{-0.10cm}
(a) \footnotesize{Fixed Step Size vs Error}
\end{minipage}
\begin{minipage}[t]{0.333\linewidth}
\centering
\includegraphics[width=0.98\textwidth]{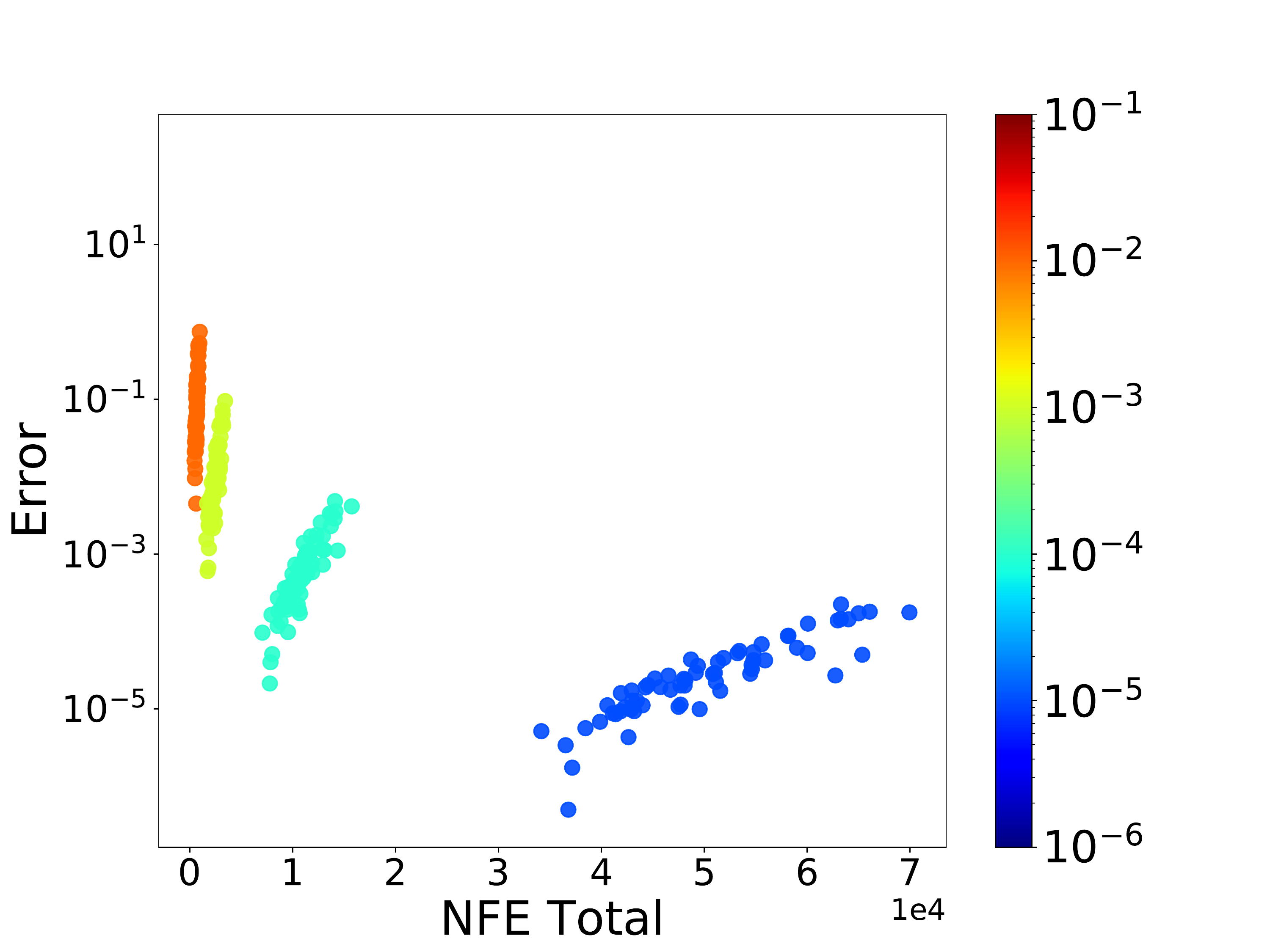} \\ \vspace{-0.10cm}
(b) \footnotesize{Total NFE vs Error}
\end{minipage}
\begin{minipage}[t]{0.333\linewidth}
\centering
\includegraphics[width=0.98\textwidth]{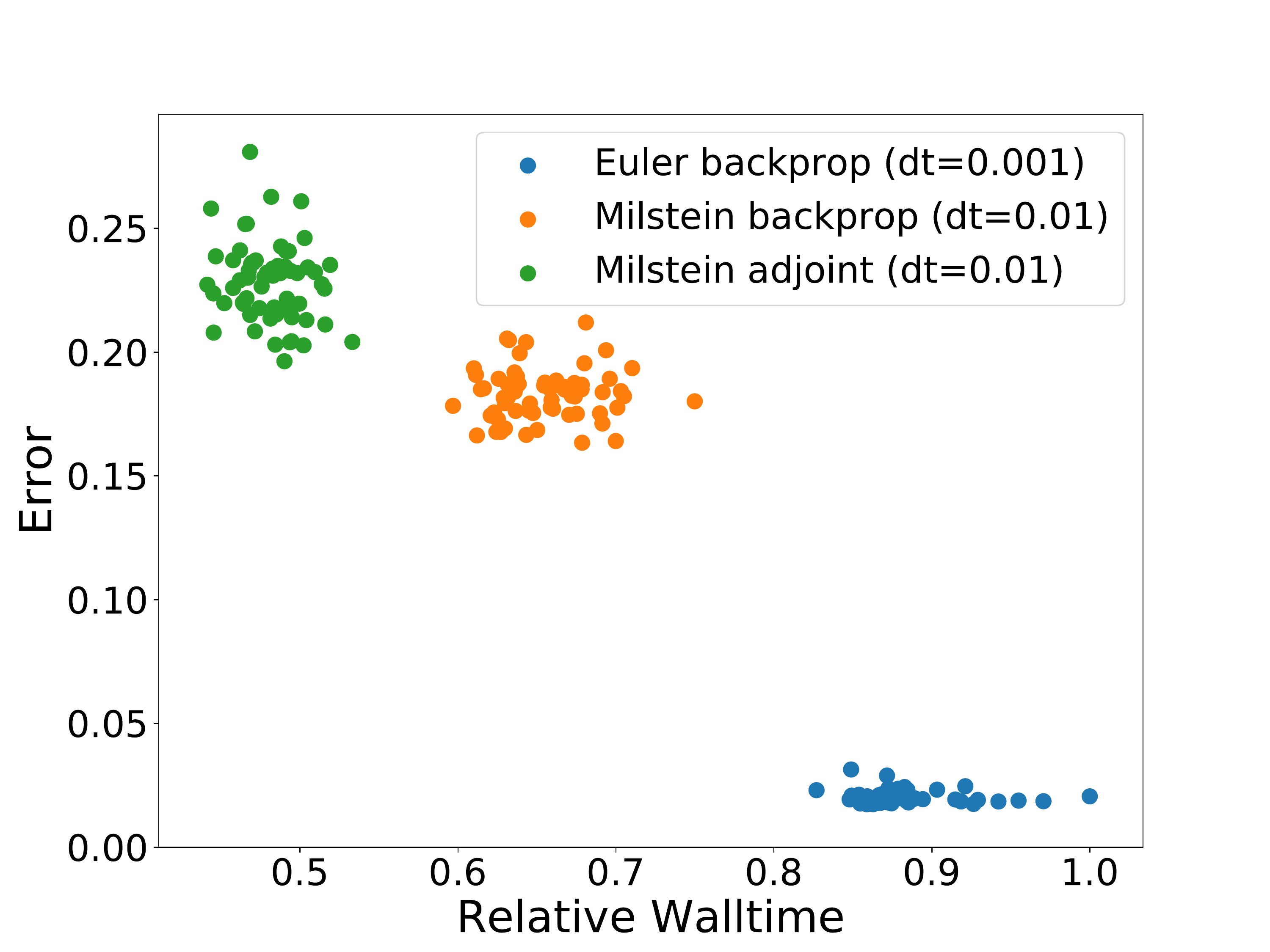} \\ \vspace{-0.10cm}
(c) \footnotesize{Efficiency Comparison}
\end{minipage}
\begin{minipage}[t]{0.333\linewidth}
\centering
{\includegraphics[width=0.98\textwidth]{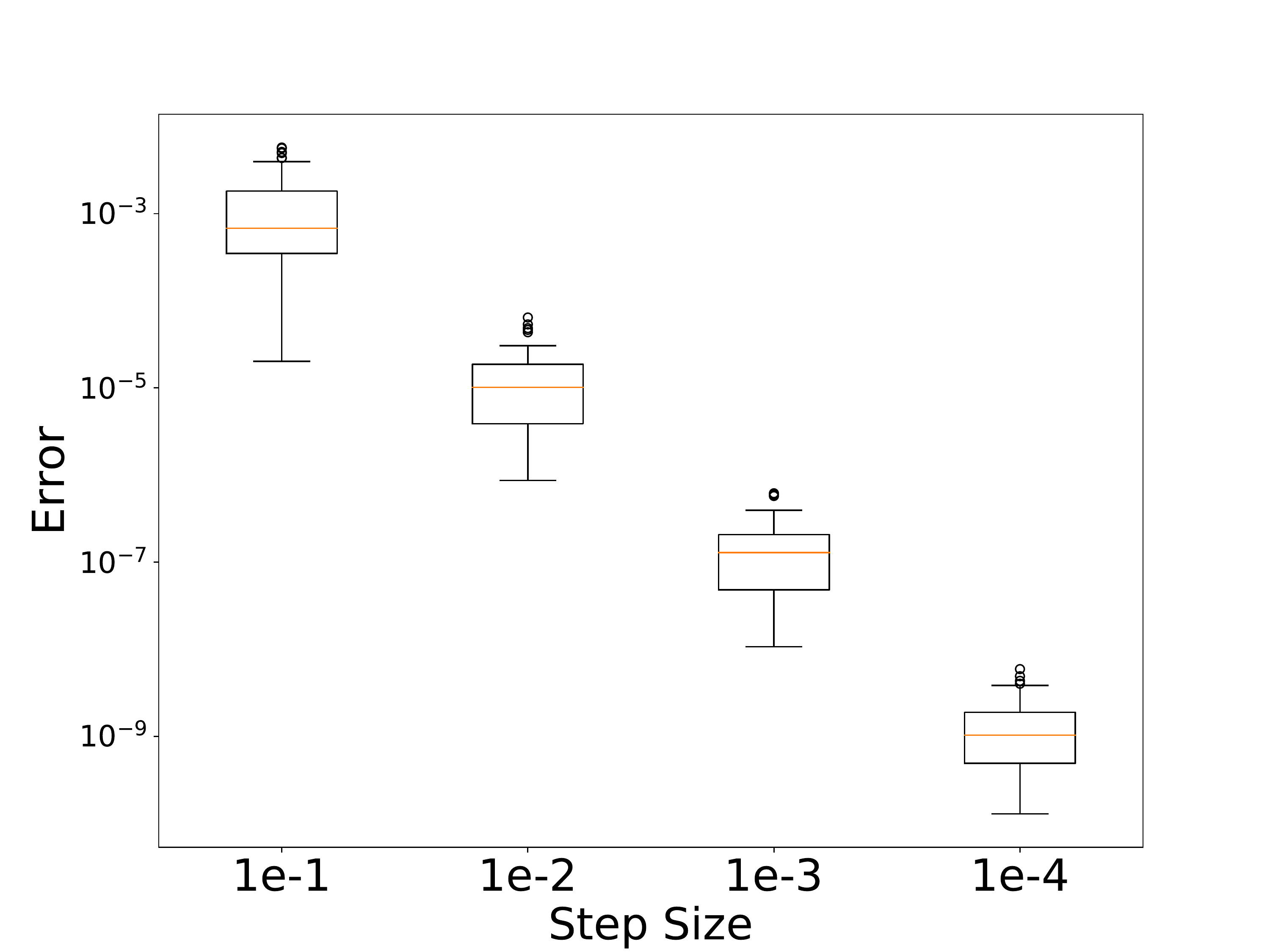}} \\ \vspace{-0.10cm}
(d) \footnotesize{Fixed Step Size vs Error}
\end{minipage}
\begin{minipage}[t]{0.333\linewidth}
\centering
\includegraphics[width=0.98\textwidth]{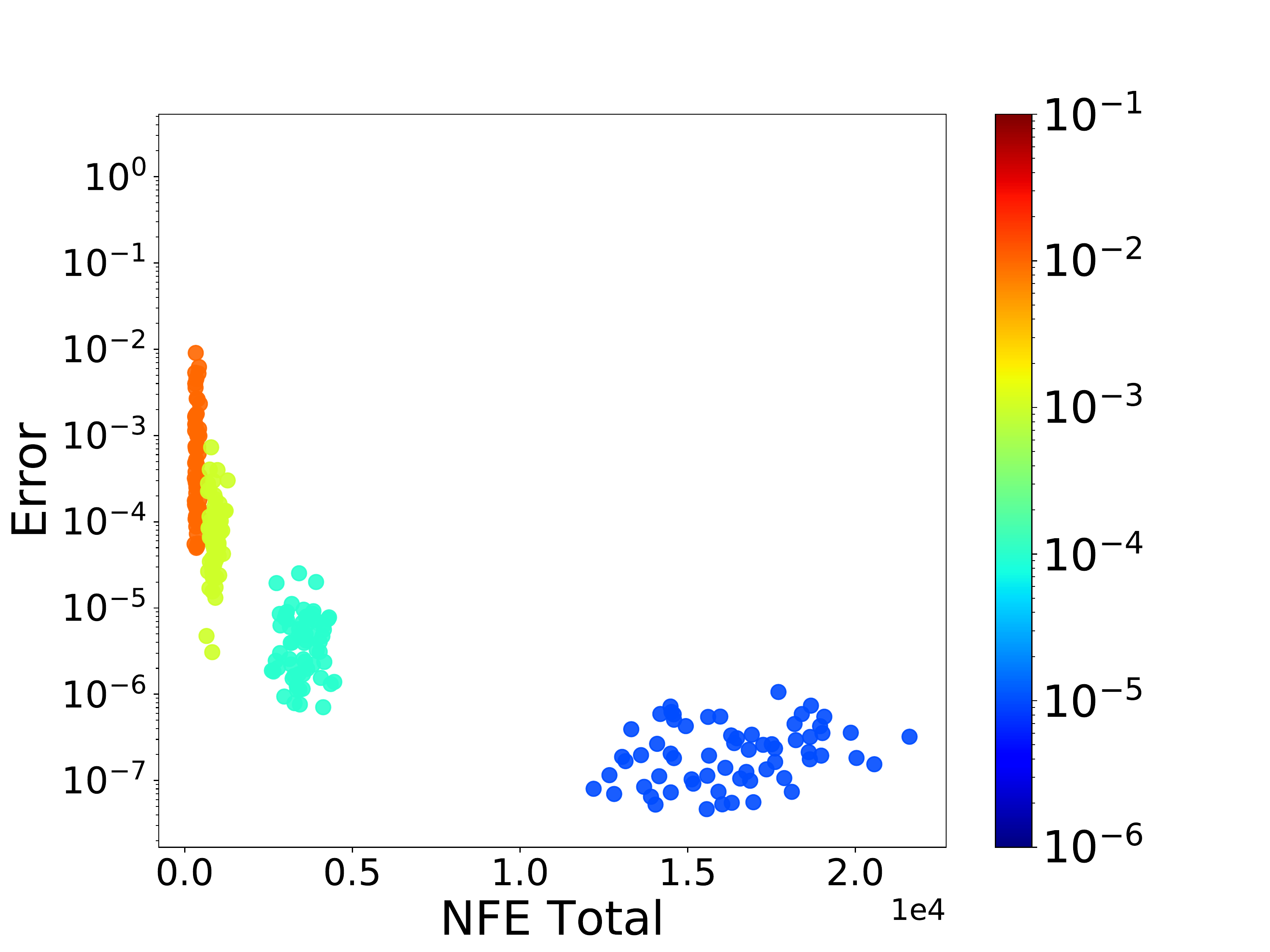} \\ \vspace{-0.10cm}
(e) \footnotesize{Total NFE vs Error}
\end{minipage}
\begin{minipage}[t]{0.333\linewidth}
\centering
\includegraphics[width=0.98\textwidth]{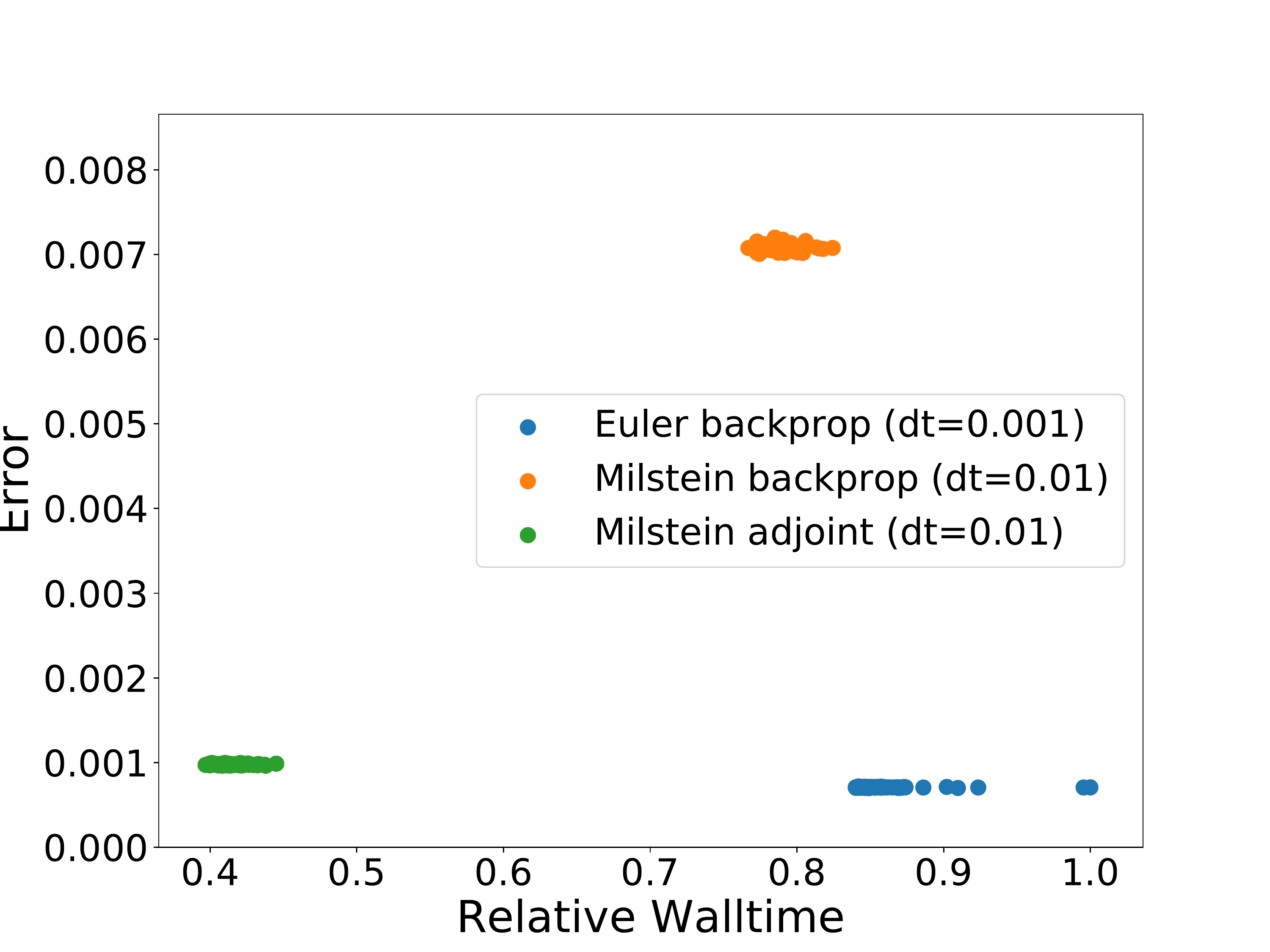} \\ \vspace{-0.10cm}
(f) \footnotesize{Efficiency Comparison}
\end{minipage}
\caption{
(a-c) Example 1. (d-f) Example 3.
}
\end{figure*}

\subsection{Toy Datasets Configuration}\label{app:toy_datasets}
\subsubsection{Geometric Brownian Motion}
Consider a geometric Brownian motion SDE:
\eq{
	\dX_t = \mu X_t \dt + \sigma X_t \dW_t, \quad X_0 = x_0.
}
We use $\mu=1$, $\sigma=0.5$, and $x_0=0.1 + \epsilon$ as the ground-truth model, where $\epsilon \sim \N(0, 0.03^2)$. We sample $1024$ time series, each of which is observed at intervals of 0.02 from time 0 to time 1. We corrupt this data using Gaussian noise with mean zero and standard deviation $0.01$.

To recover the dynamics, we use a GRU-based~\cite{cho2014learning} latent SDE model where the GRU has 1 layer and 100 hidden units, the prior and posterior drift functions are MLPs with 1 hidden layer of 100 units, and the diffusion function is an MLP with 1 hidden layer of 100 hidden units and the sigmoid activation applied at the end. 
The drift function in the posterior is time-inhomogenous in the sense that it takes in a context vector of size 1 at each observation that is output by the GRU from running backwards after processing all future observations.
The decoder is a linear mapping from a 4 dimensional latent space to observation space. 
For all nonlinearities, we use the softplus function.
We fix the observation model to be Gaussian with noise standard deviation $0.01$.

We optimize the model jointly with respect to the parameters of a Gaussian distribution for initial latent state distribution, the prior and posterior drift functions, the diffusion function, the GRU encoder, and the decoder. 
We use a fixed discretization with step size of $0.01$ in both the forward and backward pass.
We use the Adam optimizer~\cite{kingma2014adam} with an initial learning rate of $0.01$ that is decay by a factor of $0.999$ after each iteration.
We use a linear KL annealing schedule over the first 50 iterations. 

\subsubsection{Stochastic Lorenz Attractor}
Consider a stochastic Lorenz attractor SDE with diagonal noise:
\eq{
	\dX_t=& \sigma\left(Y_t-X_t \right) \dt + \alpha_x \dW_t, \quad X_0 = x_0, \\
	\dY_t=& \bracks{ X_t \left(\rho-Z_t \right)-Y_t } \dt + \alpha_y \dW_t, \quad Y_0 = y_0, \\
	\dZ_t=& \bracks{ X_t Y_t - \beta Z_t} \dt + \alpha_z \dW_t, \quad Z_0 = z_0. 
}
We use $\sigma=10$, $\rho=28$, $\beta=8/3$, $(\alpha_x, \alpha_y, \alpha_z) = (.15, .15., .15)$, and $(x_0, y_0, z_0)$ sampled from the standard Gaussian distribution as the ground-truth model. We sample $1024$ time series, each of which is observed at intervals of $0.025$ from time 0 to time 1. We normalize these samples by their mean and standard deviation across each dimension and corrupt this data by Gaussian noise with mean zero and standard deviation $0.01$.

We use the same architecture and training procedure for the latent SDE model as in the geometric Brownian motion section, except that the diffusion function consists of four small neural networks, each for a single dimension of the latent SDE.

\subsection{Additional Visualization}\label{app:visualization}

\begin{figure*}[ht]
\begin{minipage}[ht]{\linewidth}
\centering
{\includegraphics[width=0.98\textwidth, clip, trim=4mm 0mm 4mm 0mm]{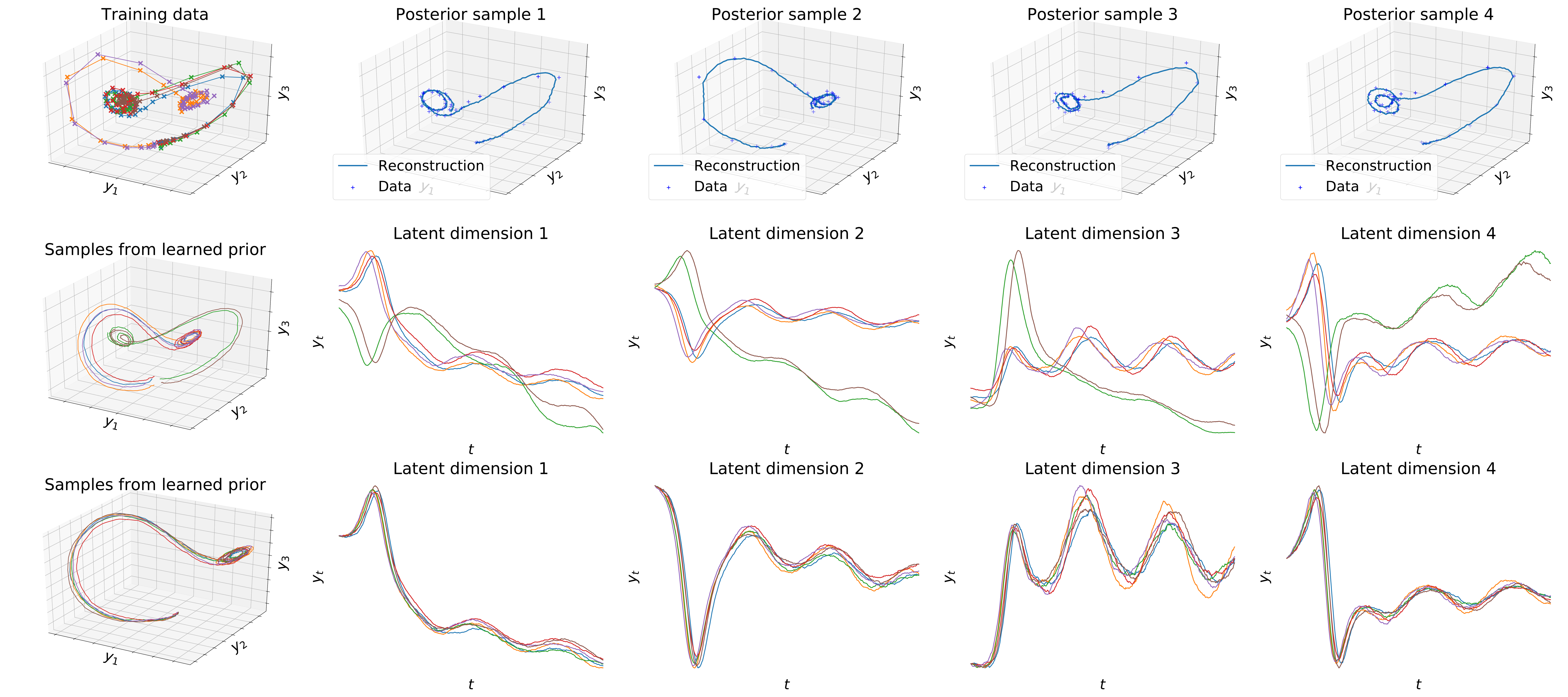}} \\ \vspace{-0.10cm}
\end{minipage}
\caption{
Additional visualizations of learned posterior and prior dynamics on the synthetic stochastic Lorenz attractor dataset.
First row displays the true data and posterior reconstructions. 
Second row displays samples with initial latent state for each trajectory is sampled independently. 
Third row displays samples with initial latent state sampled and fixed to be the same for different trajectories. 
}
\label{fig:toy_experiments_additional_lorenz}
\end{figure*}

\begin{figure*}[ht]
\begin{minipage}[ht]{\linewidth}
\centering
{\includegraphics[width=0.98\textwidth]{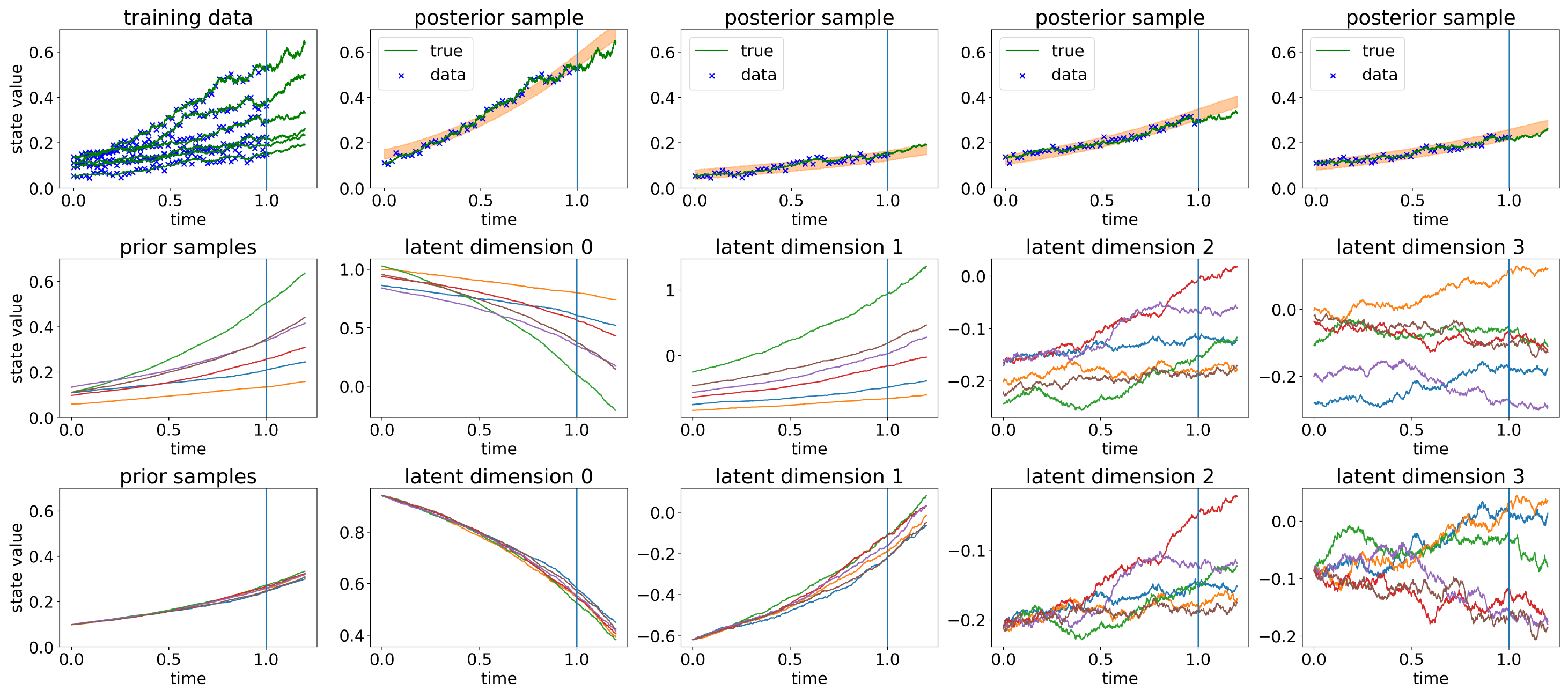}} \\ \vspace{-0.10cm}
\end{minipage}
\caption{
Visualizations of learned posterior and prior dynamics on the synthetic geometric Brownian motion dataset.
First row displays the true data and posterior reconstructions. Orange contour covers $95\%$ of 512 samples.
Second row displays samples with initial latent state for each trajectory is sampled independently. 
Third row displays samples with initial latent state sampled and fixed to be the same for different trajectories. 
}
\label{fig:toy_experiments_additional_gbm}
\end{figure*}

See Figure~\ref{fig:toy_experiments_additional_lorenz} for additional visualization on the synthetic Lorenz attractor dataset.
See Figure~\ref{fig:toy_experiments_additional_gbm} for visualization on the synthetic geometric Brownian motion dataset.
We comment that for the second example, the posterior reconstructs the data well, and the prior process exhibit behavior of the data. However, from the third row, we can observe that the prior process is learned such that most of the uncertainty is account for in the initial latent state. We leave the investigation of more interpretable prior process for future work.

\begin{figure}[ht]
\begin{minipage}[t]{\linewidth}
\centering
{\includegraphics[width=0.8\textwidth]{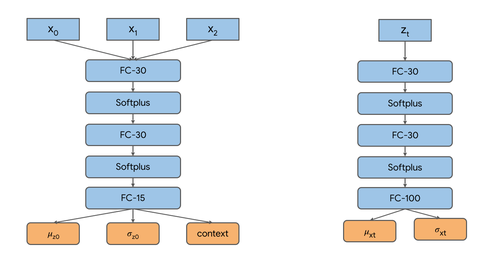}}
\\ \vspace{-0.10cm}
\end{minipage}
\begin{minipage}[t]{\linewidth}
\centering
\includegraphics[width=0.8\textwidth]{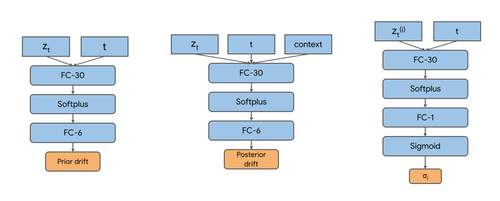}
\\ \vspace{-0.10cm}
\end{minipage}
\caption{
Architecture specifics for the latent SDE model used to train on the mocap dataset.
First row from left to right are the encoder and decoder. 
Second row from left to right are the prior drift, posterior drift, and diffusion functions.
}
\label{fig:architecture}
\end{figure}

\subsection{Model Architecture for Learning from Motion Capture Dataset}\label{app:architecture}
We use a latent SDE model with an MLP encoder which takes in the first three frames and outputs the mean and log-variance of the variational distribution of the initial latent state and a context vector. 
The decoder has a similar architecture as that for the ODE$^2$VAE model~\cite{yildiz2019ode} and projects the $6$-dimensional latent state into the $50$-dimensional observation space. 
The posterior drift function takes in a $3$-dimensional context vector output by the encoder and the current state and time, whereas the prior drift only takes in the current state and time. 
The diffusion function is composed of multiple small neural nets, each producing a scalar for the corresponding dimension such that the posterior SDE has diagonal noise. 
We use the same observation likelihood as that of the ODE$^2$VAE model~\cite{yildiz2019ode}.
We comment that the overall parameter count of our model (11605) is smaller than that of ODE$^2$VAE for the same task (12157).

The latent ODE baseline was implemented with a similar architecture, except is does not have the diffusion and prior drift components, and its vector field defining the ODE does not take in a context vector. Therefore, the model has slightly fewer parameters (10573) than the latent SDE model. See Figure~\ref{fig:architecture} for overall details of the architecture.

The main hyperparameter we tuned was the coefficient for reweighting the KL. For both the latent ODE and SDE, we considered training the model with a reweighting coefficient in $\{1, 0.1, 0.01, 0.001\}$, either with or without a linear KL annealing schedule that increased from $0$ to the prescribed value over the first $200$ iterations of training.

\clearpage
\subsection{Stochastic Adjoint Implementation}\label{app:adjoint_code}
We include the core implementation of the stochastic adjoint, assuming access to a callable Brownian motion \texttt{bm}, an Euler-Maruyama integrator \texttt{ito\_int\_diag} for diagonal noise SDEs, and several helper functions whose purposes can be inferred from their names.

\begin{minted}[fontsize=\footnotesize,]{python}
class _SdeintAdjointMethod(torch.autograd.Function):

  @staticmethod
  def forward(ctx, *args):
    (y0, f, g, ts, flat_params_f, flat_params_g, dt, bm) = (
      args[:-8], args[-7], args[-6], args[-5], args[-4], args[-3], args[-2], args[-1])
    ctx.f, ctx.g, ctx.dt, ctx.bm = f, g, dt, bm

    def g_prod(t, y, noise):
      g_eval = g(t=t, y=y)
      g_prod_eval = tuple(
        g_eval_i * noise_i for g_eval_i, noise_i in _zip(g_eval, noise))
      return g_prod_eval

    with torch.no_grad():
      ans = ito_int_diag(f, g_prod, y0, ts, dt, bm)
    ctx.save_for_backward(ts, flat_params_f, flat_params_g, *ans)
    return ans

  @staticmethod
  def backward(ctx, *grad_outputs):
    ts, flat_params_f, flat_params_g, *ans = ctx.saved_tensors
    f, g, dt, bm = ctx.f, ctx.g, ctx.dt, ctx.bm
    f_params, g_params = tuple(f.parameters()), tuple(g.parameters())
    n_tensors = len(ans)

    def aug_f(t, y_aug):
      y, adj_y = y_aug[:n_tensors], y_aug[n_tensors:2 * n_tensors]

      with torch.enable_grad():
        y = tuple(y_.detach().requires_grad_(True) for y_ in y)
        adj_y = tuple(adj_y_.detach() for adj_y_ in adj_y)

        g_eval = g(t=-t, y=y)
        gdg = torch.autograd.grad(
          outputs=g_eval, inputs=y,
          grad_outputs=g_eval,
          create_graph=True)
        f_eval = f(t=-t, y=y)
        f_eval = _sequence_subtract(gdg, f_eval)  # -f + gdg.

        vjp_y_and_params = torch.autograd.grad(
          outputs=f_eval, inputs=y + f_params + g_params,
          grad_outputs=tuple(-adj_y_ for adj_y_ in adj_y),
          retain_graph=True, allow_unused=True)
        vjp_y = vjp_y_and_params[:n_tensors]
        vjp_f = vjp_y_and_params[-len(f_params + g_params):-len(g_params)]
        vjp_g = vjp_y_and_params[-len(g_params):]

        vjp_y = tuple(torch.zeros_like(y_) 
          if vjp_y_ is None  else vjp_y_ for vjp_y_, y_ in zip(vjp_y, y))

        adj_times_dgdx = torch.autograd.grad(
          outputs=g_eval, inputs=y,
          grad_outputs=adj_y,
          create_graph=True)
        extra_vjp_y_and_params = torch.autograd.grad(
          outputs=g_eval, inputs=y + f_params + g_params,
          grad_outputs=adj_times_dgdx,
          allow_unused=True)
        extra_vjp_y = extra_vjp_y_and_params[:n_tensors]
        extra_vjp_f = extra_vjp_y_and_params[-len(f_params + g_params):-len(g_params)]
        extra_vjp_g = extra_vjp_y_and_params[-len(g_params):]

        extra_vjp_y = tuple(
          torch.zeros_like(y_) if extra_vjp_y_ is None 
          else extra_vjp_y_ for extra_vjp_y_, y_ in zip(extra_vjp_y, y))

        vjp_y = _sequence_add(vjp_y, extra_vjp_y)
        vjp_f = vjp_f + extra_vjp_f
        vjp_g = vjp_g + extra_vjp_g

      return (*f_eval, *vjp_y, vjp_f, vjp_g)

    def aug_g_prod(t, y_aug, noise):
      y, adj_y = y_aug[:n_tensors], y_aug[n_tensors:2 * n_tensors]

      with torch.enable_grad():
        y = tuple(y_.detach().requires_grad_(True) for y_ in y)
        adj_y = tuple(adj_y_.detach() for adj_y_ in adj_y)

        g_eval = tuple(-g_ for g_ in g(t=-t, y=y))
        vjp_y_and_params = torch.autograd.grad(
          outputs=g_eval, inputs=y + f_params + g_params,
          grad_outputs=tuple(-noise_ * adj_y_ for noise_, adj_y_ in zip(noise, adj_y)),
          allow_unused=True)
        vjp_y = vjp_y_and_params[:n_tensors]
        vjp_f = vjp_y_and_params[-len(f_params + g_params):-len(g_params)]
        vjp_g = vjp_y_and_params[-len(g_params):]

        vjp_y = tuple(
          torch.zeros_like(y_) if vjp_y_ is None 
          else vjp_y_ for vjp_y_, y_ in zip(vjp_y, y)
        )
        g_prod_eval = _sequence_multiply(g_eval, noise)

      return (*g_prod_eval, *vjp_y, vjp_f, vjp_g)

    def aug_bm(t):
      return tuple(-bmi for bmi in bm(-t))

    T = ans[0].size(0)
    with torch.no_grad():
      adj_y = tuple(grad_outputs_[-1] for grad_outputs_ in grad_outputs)
      adj_params_f = torch.zeros_like(flat_params_f)
      adj_params_g = torch.zeros_like(flat_params_g)

      for i in range(T - 1, 0, -1):
        ans_i = tuple(ans_[i] for ans_ in ans)
        aug_y0 = (*ans_i, *adj_y, adj_params_f, adj_params_g)
        aug_ans = ito_int_diag(
          f=aug_f, g_prod=aug_g_prod, y0=aug_y0,
          ts=torch.tensor([-ts[i], -ts[i - 1]]).to(ts),
          dt=dt, bm=aug_bm)
        adj_y = aug_ans[n_tensors:2 * n_tensors]
        adj_params_f, adj_params_g = aug_ans[-2], aug_ans[-1]

        # Take the result at the end time.
        adj_y = tuple(adj_y_[1] for adj_y_ in adj_y)
        adj_params_f, adj_params_g = adj_params_f[1], adj_params_g[1]

        # Accumulate gradients at intermediate points.
        adj_y = _sequence_add(
          adj_y, tuple(grad_outputs_[i - 1] for grad_outputs_ in grad_outputs)
        )
    return (*adj_y, None, None, None, adj_params_f, adj_params_g, None, None)
\end{minted}

\end{document}